\PassOptionsToPackage{table}{xcolor}
\documentclass{article}

\usepackage[preprint]{neurips_2026}

\usepackage[most]{tcolorbox}

\tcbset{
  theoremstyle/.style={
    enhanced,
    colback=gray!8,
    colframe=gray!40,
    boxrule=0pt,
    arc=3pt,
    left=6pt,
    right=6pt,
    top=6pt,
    bottom=6pt,
    boxsep=0pt,
    breakable
  }
}

\newtcbtheorem[number within=section]{propositionbox}{Proposition}%
{theoremstyle}{prop}
\usepackage{amssymb}
\usepackage[utf8]{inputenc} %
\usepackage[T1]{fontenc}    %
\PassOptionsToPackage{final,pagebackref=true}{hyperref}
\usepackage{hyperref}       %
\usepackage{url}            %
\usepackage{booktabs}       %
\usepackage{graphicx}       %
\usepackage{wrapfig}        %
\usepackage{amsmath}        %
\usepackage{amsfonts}       %
\usepackage{amsthm}         %
\usepackage{nicefrac}       %
\usepackage{microtype}      %
\usepackage{xcolor}  %
\usepackage{multirow}
\usepackage{threeparttable}

\colorlet{tablegray}{black!8}

\usepackage{siunitx}
\usepackage{array}
\usepackage{makecell}
\usepackage{algorithm}
\usepackage{algpseudocode}
\usepackage{float}
\usepackage{xspace}

\hypersetup{
	colorlinks=true,       %
	linkcolor=blue,        %
	citecolor=blue,        %
	filecolor=magenta,     %
	urlcolor=blue
}

\newcolumntype{L}[1]{>{\raggedright\arraybackslash}p{#1}}

\newtheorem{proposition}{Proposition}

\makeatletter
\newcommand{\appendixcontents}{%
  \begingroup
  \section*{Appendix Contents}
  \small
  \vspace{-0.5em}
  \setlength{\parskip}{0pt}
  \@starttoc{apc}
  \endgroup
  \newpage
}
\newcommand*\l@appendixsection{\@dottedtocline{1}{0em}{2.4em}}
\newcommand*\l@appendixsubsection{\@dottedtocline{2}{1.5em}{3.2em}}
\newcommand{\startappendixcontents}{%
  \let\paper@section\section
  \let\paper@subsection\subsection
  \renewcommand{\section}{\@ifstar{\paper@section*}{\paper@app@section}}%
  \renewcommand{\subsection}{\@ifstar{\paper@subsection*}{\paper@app@subsection}}%
}
\newcommand{\paper@app@section}[2][]{%
  \if\relax\detokenize{#1}\relax
    \paper@section{#2}%
    \addcontentsline{apc}{appendixsection}{\protect\numberline{\thesection}#2}%
  \else
    \paper@section[#1]{#2}%
    \addcontentsline{apc}{appendixsection}{\protect\numberline{\thesection}#1}%
  \fi
}
\newcommand{\paper@app@subsection}[2][]{%
  \if\relax\detokenize{#1}\relax
    \paper@subsection{#2}%
    \addcontentsline{apc}{appendixsubsection}{\protect\numberline{\thesubsection}#2}%
  \else
    \paper@subsection[#1]{#2}%
    \addcontentsline{apc}{appendixsubsection}{\protect\numberline{\thesubsection}#1}%
  \fi
}
\makeatother

\newcommand{\MDLM}{MDLM~\citep{sahoo2024mdlm}\xspace}
\newcommand{\UDLM}{UDLM~\citep{schiff2025discreteguidance}\xspace}
\newcommand{\PAIRFLOW}{PAIRFLOW~\citep{park2026pairflow}\xspace}
\newcommand{\DCD}{DCD~\citep{sahoo2025diffusionduality}\xspace}
\newcommand{\ReDi}{ReDi~\citep{yoo2025redi}\xspace}
\newcommand{\CFM}{CFM~\citep{roos2026categoricalflowmaps}\xspace}
\newcommand{\DiscFM}{Discrete Flow Maps~\citep{potaptchik2026discreteflowmaps}\xspace}
\newcommand{\FlowMapLM}{FMLM~\citep{lee2026flm}\xspace}

\newcommand{\SFMOT}{SFM w/ OT~\citep{cheng2024sfm}\xspace}
\newcommand{\DirichletFM}{DirichletFM~\citep{stark2024dirichletfm}\xspace}
\newcommand{\SLM}{SLM~\citep{song2025slm}\xspace}
\newcommand{\FisherFlow}{Fisher-Flow~\citep{davis2024fisherflow}\xspace}

\newcommand{\FMLM}{FMLM~\citep{lee2026flm}\xspace}

\newcommand{\DDSM}{DDSM~\citep{avdeyev2023ddsm}\xspace}
\newcommand{\DThreePM}{D3PM~\citep{austin2021d3pm}\xspace}
\newcommand{\DFM}{DFM~\citep{gat2024discrete}\xspace}

\newcommand{\namelong}{Coupling Model\xspace}
\newcommand{\nameshort}{\emph{Coupling Model}\xspace}
\newcommand{\namelongs}{Coupling Models\xspace}
\newcommand{\nameshorts}{Coupling Models\xspace}
\newcommand{\xhdr}[1]{{\noindent\bfseries #1}.}
\newcommand{\cut}[1]{}

\title{Coupling Models for One-Step Discrete Generation}

\author{
Fred Zhangzhi Peng\thanks{Correspondence to: \texttt{zhangzhi.peng@duke.edu}}\\
Duke University\\
\And
Avishek Joey Bose\\
Imperial College London\\
\AND
Anru R. Zhang\\
Duke University\\
\And
Alexander Tong\\
AITHYRA
}

\begin{document}

\maketitle

\begin{abstract}
\looseness=-1
Generative modeling over discrete structures underpins applications across deep learning, from biological sequence design and code generation to large language models, yet generation often remains sequential, relying on autoregressive decoding or iterative refinement. In this work, we introduce \emph{\namelongs} (\nameshorts), a one-step discrete generative model that learns a direct coupling between discrete sequences and Gaussian latents. Unlike recent distillation methods that compress a pretrained multi-step sampler into a few steps, \nameshort trains a purpose-built decoder to invert this coupling and generate samples in a single step. The model also avoids complex continuous flows over the simplex and hand-specified data-to-noise couplings. Empirically, \nameshort improves the strongest one-step baselines in each domain: it reduces LM1B text-generation perplexity by 33\% at its lowest-perplexity operating point, Fly Brain enhancer-design FBD by 18\%, and MNIST-Binary FID by 46\%. These results suggest that effective one-step discrete generation depends strongly on how data and noise are coupled before decoding.
Code is available at \url{https://github.com/pengzhangzhi/Coupling-Models}.

\end{abstract}
\vspace{-1.0em}
\begin{figure}[H]
  \centering
  \includegraphics[width=0.60\linewidth]{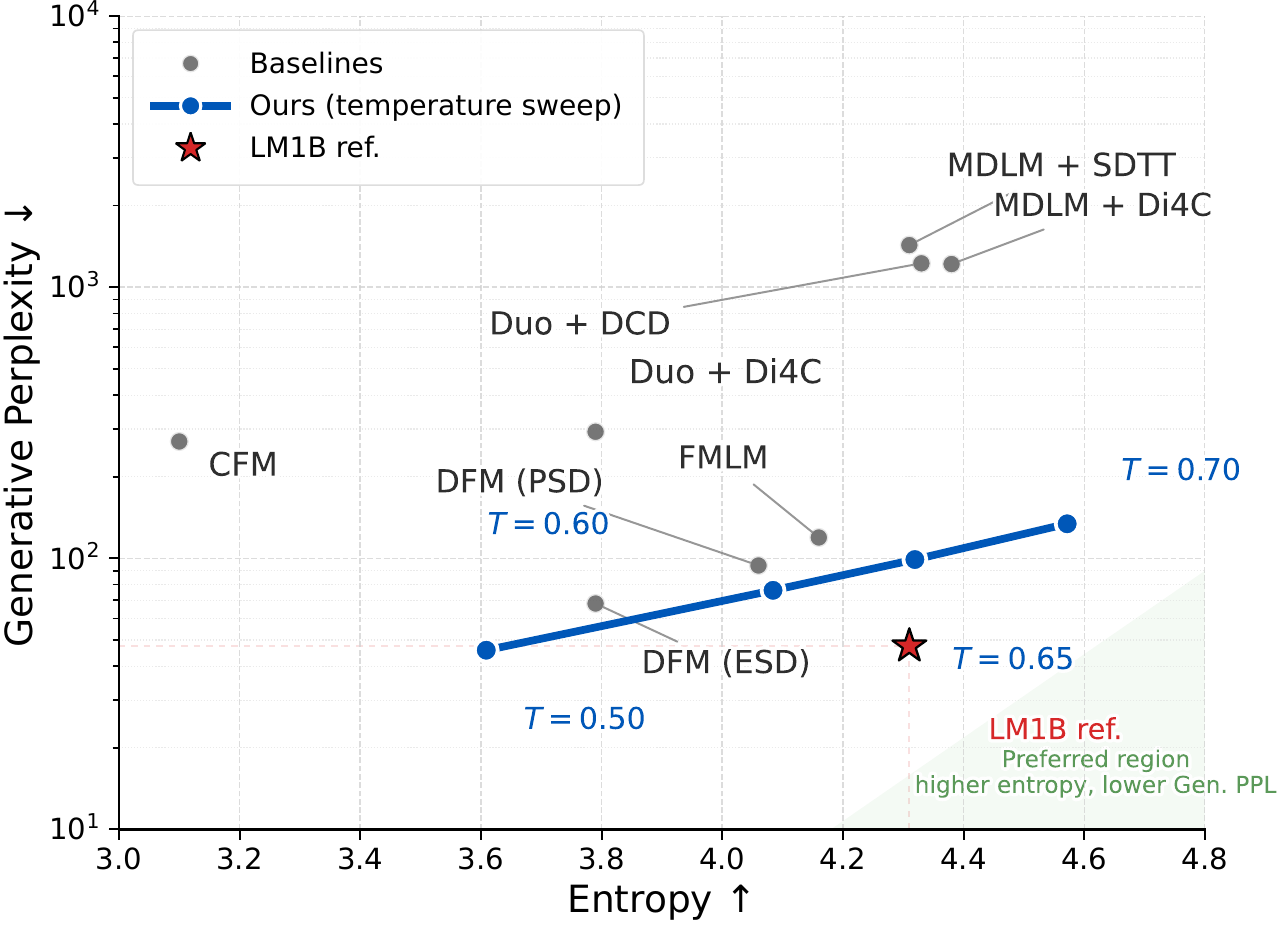}
  \caption{
    \looseness=-1 \textbf{Toward a one-step language generation quality--diversity frontier.}
    On LM1B, \nameshort shifts the one-step entropy--perplexity frontier toward the desired low-perplexity, high-entropy region: it improves over prior one-step baselines in generative perplexity while preserving entropy close to the real-data reference.
    }
  \label{fig:lm1b-entropy-genppl-tradeoff}
\end{figure}

\section{Introduction}
\label{sec:introduction}

\looseness=-1
Discrete generation is central to language modeling, code generation, biological sequence design, and discrete image modeling~\citep{brown2020language,dauparas2022robust,peng2025pathplanning,stark2024dirichletfm,davis2024fisherflow}. The strongest models in these domains still rely on inference procedures that use multiple model evaluations to represent complex joint structure. For instance, autoregressive models factorize the sequence distribution into next-token conditionals and remain the dominant paradigm for modern large language models~\citep{vaswani2017attention,brown2020gpt3,touvron2023llama2,yang2024qwen2}. Although this factorization is expressive and scalable, generation is inherently serial. Discrete diffusion, masked refinement, and discrete flow models offer a more parallel alternative by updating many positions simultaneously, but they also require repeated steps to achieve higher levels of coherence~\citep{austin2021d3pm,sahoo2024mdlm,lou2024sedd,zheng2024reparameterized,nie2025scalingmdm,gong2025diffullama,schiff2025discreteguidance,gat2024discrete,cheng2024sfm,stark2024dirichletfm,davis2024fisherflow}. %

\looseness=-1
The search for faster generation motivates the more extreme setting of \emph{true one-step} discrete generation: a single prior sample coupled with a single model forward pass to produce the full sequence. A successful one-step generator would remove the sequential bottleneck of autoregressive decoding and the repeated refinement cost of discrete diffusion-style models. Toward this goal, recent work has made the regime more feasible through consistency, distillation, self-distillation, or rectification~\citep{sahoo2025diffusionduality,yoo2025redi,deschenaux2025beyond,chen2025dlmone,zhu2025dimo,monsefi2025fsdfm,zheng2026didi,zhang2026t3d,bartosh2026fldd}. Another line improves the low-step trajectory itself through better source--target couplings and discrete flow constructions~\citep{gat2024discrete,park2026pairflow}. A third line moves generation into continuous or simplex-valued surrogate spaces, where categorical variables can be modeled through smoother geometric structure~\citep{cheng2024sfm,stark2024dirichletfm,davis2024fisherflow}.

\looseness=-1
A particularly active recent direction is based on \emph{flow maps}. Rather than simulating a full flow or denoising trajectory, flow-map methods learn endpoint maps that directly predict terminal states. Recent discrete-generation work adapts this idea to categorical data and language. For instance, \DiscFM{} formulates discrete flow maps by aligning endpoint-map training with the geometry of the probability simplex. \FMLM{} learns continuous denoising flows over one-hot token embeddings and distills them into flow maps for efficient language generation. \CFM{} defines categorical flow maps toward the simplex and trains them with self-distillation and endpoint-consistency objectives. Although these methods differ in parameterization, objectives, and empirical domains, they share a common premise: fast discrete generation can be obtained by learning better endpoint maps for continuous or simplex-valued generative trajectories.%
This progress raises a basic question:

\textit{\quad\quad \textbf{Q.} Is trajectory compression the
most natural route to one-step discrete generation?}

\looseness=-1
We argue that the difficulty of one-step discrete generation is not only an optimization problem, but also a representation problem. A direct one-step categorical decoder must recover global cross-token dependence in a single parallel prediction, without the benefit of autoregressive factorization or iterative denoising. As the number of refinement steps drops, the model has fewer opportunities to repair local decisions using global context, and quality can degrade even when the underlying multi-step model is capable~\citep{sahoo2025diffusionduality}. The central issue is therefore not only how to shorten an existing trajectory, but also what representation makes the one-step map easier to learn in the first place.

\looseness=-1
\xhdr{Present work}
We propose \namelongs (\nameshorts), a coupling-based construction for one-step discrete generation. Instead of distilling an iterative generator or learning an endpoint map over a discrete or simplex-valued trajectory, we first learn an easy-to-sample coupling between discrete sequences and Gaussian latents. This coupling induces paired supervision between Gaussian latents and target sequences. We then train a parallel decoder to invert the coupling and emit all output tokens in one forward pass. At inference time, generation requires one step: draw $z\sim\mathcal{N}(0,I)$ and generate.

\looseness=-1
We evaluate \nameshort across three regimes: (1) controlled binary image modeling on MNIST-Binary, (2) biological sequence design on DNA enhancers, and (3) open-ended language generation on LM1B. In the strict one-step setting, \nameshort improves the best available one-step comparisons across all three domains. On MNIST-Binary, it reduces FID to $5.50$, outperforming the strongest one-step baselines in Table~\ref{tab:mnist_binary_main_1col}. On DNA enhancer generation, it improves over the best one-step baseline, distilled Dirichlet FM~\citep{stark2024dirichletfm}, reducing FB FBD from $15.8$ to $12.9\pm0.3$ in Table~\ref{tab:dna_enhancer_clean}. On LM1B, it improves over the flow-map-based methods \CFM{}, \DiscFM{}, and \FMLM{} in generative perplexity while maintaining entropy near the real-data reference. The central language-generation result is that \nameshort does not obtain lower perplexity by collapsing diversity; it shifts the one-step entropy--perplexity frontier toward the desired low-perplexity, high-entropy region, as depicted in Figure~\ref{fig:lm1b-entropy-genppl-tradeoff}.
Finally, we demonstrate that beyond unconditional sampling, the same latent interface makes guidance more direct and efficient. More precisely, classifier-free guidance acts on the final logits, classifier guidance optimizes the latent variable, and reward fine-tuning updates the generator, avoiding the long stochastic control path required by multi-step denoising. Empirically, this shorter control path improves the quality--efficiency tradeoff relative to the masked-diffusion baseline while using far fewer function evaluations.

\section{Related Work}
\label{sec:related_work}

\xhdr{Autoregressive, diffusion, and flow-based discrete generation}
Discrete generation is commonly approached through inference procedures that use multiple model evaluations to represent complex joint structure. Autoregressive models factorize the sequence distribution into conditional next-token predictions and remain the dominant paradigm for language and sequence modeling~\citep{vaswani2017attention,brown2020gpt3,touvron2023llama2,yang2024qwen2}. This factorization is expressive and scalable, but generation is inherently serial. Discrete diffusion and masked-refinement models update many positions in parallel, but recover global dependence through repeated denoising or replacement steps~\citep{austin2021d3pm,sahoo2024mdlm,lou2024sedd,zheng2024reparameterized,nie2025scalingmdm,gong2025diffullama,schiff2025discreteguidance}. Related discrete flow and simplex-based methods formulate generation through transport or flow matching over categorical variables or continuous relaxations of the simplex~\citep{gat2024discrete,cheng2024sfm,stark2024dirichletfm,davis2024fisherflow}. Together, these methods define the main speed-quality tradeoff studied in this paper: strong discrete generation typically relies on either serial factorization or iterative refinement. A more detailed discussion appears in Appendix~\ref{app:related_work}.

\xhdr{Fast one-step and few-step generation}
Recent work has made one-step and few-step discrete generation increasingly practical. Distillation, consistency, self-distillation, and rectification methods compress or improve multi-step generators so that fewer model evaluations are required at test time~\citep{sahoo2025diffusionduality,yoo2025redi,deschenaux2025beyond,chen2025dlmone,zhu2025dimo,monsefi2025fsdfm,zheng2026didi,zhang2026t3d,bartosh2026fldd}. Coupling-based methods such as \PAIRFLOW{} improve the low-step coupling itself by constructing better source-target couplings~\citep{park2026pairflow}. A particularly relevant line is flow-map-based generation: \DiscFM{}, \FlowMapLM{}, and \CFM{} learn endpoint maps for continuous or simplex-valued trajectories, enabling faster generation by bypassing long integration or denoising chains~\citep{potaptchik2026discreteflowmaps,lee2026flm,roos2026categoricalflowmaps}. Our method targets the same one-step regime but differs in construction. Rather than distilling an iterative generator, rectifying a discrete trajectory, or learning an endpoint map along a simplex-valued path, we first learn a coupling between continuous sequence representations and Gaussian latents. The resulting latent-sequence pairs provide supervised training data for a parallel one-step decoder. BiFlow is related in its use of normalizing flows to connect data and Gaussian noise bidirectionally, but it operates in a continuous generative setting, whereas we use the learned coupling to induce latent--sequence supervision for one-step discrete decoding~\citep{lu2026bidirectional}. Our method is best viewed as a coupling-based route to one-step discrete generation, rather than another trajectory-compression method.
\begin{figure}[t]
  \centering
  \includegraphics[width=0.72\linewidth]{\detokenize{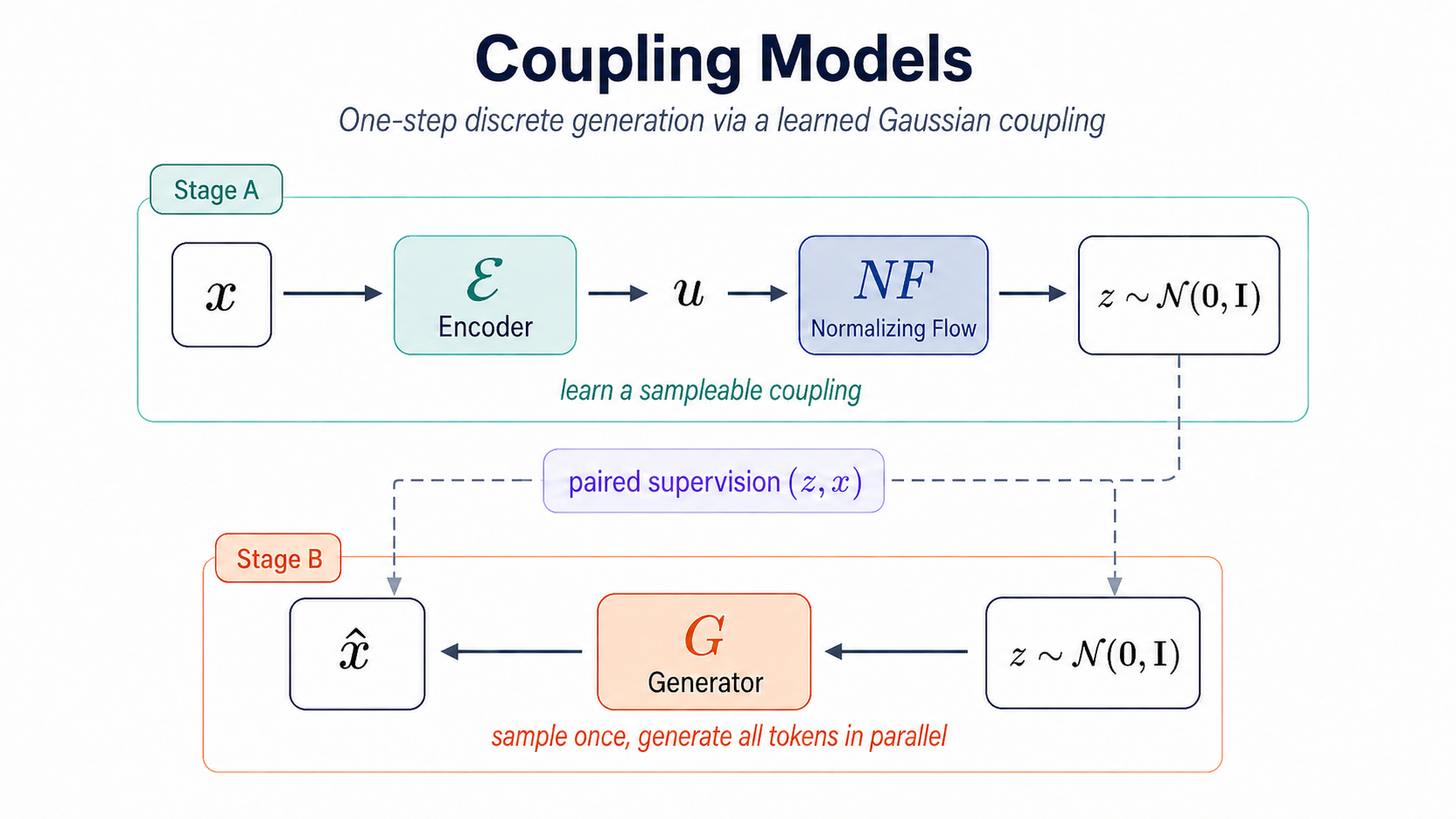}}
  \vspace{-0.7em}
  \caption{\textbf{Overview of the \namelong. }Stage~A learns a coupling from discrete sequences to Gaussian latents. Stage~B trains a parallel decoder on the induced pairs, enabling generation from a single model evaluation.}
  \label{fig:method}
\end{figure}

\section{Method}
\label{sec:method}

\looseness=-1
We study one-step generation of a discrete sequence
$x=(x_1,\ldots,x_T)$, where $x_t\in\{1,\ldots,V\}$. The main challenge is not
only computational, but representational. A direct one-step categorical decoder
predicts one distribution per position and samples all positions in parallel.
Without an additional shared variable or refinement process, this induces a
factorized sequence law, $p(x)=\prod_{t=1}^{T}p_t(x_t)$, which cannot represent
general cross-token dependence. Autoregressive models avoid this barrier through
sequential conditioning, while diffusion and masked-refinement models use
multiple denoising steps to exchange information across positions. The strict
one-step setting removes both mechanisms. Thus, the central question is what
representation can carry global sequence-level information into a single
parallel categorical emission.

\subsection{\namelong for one-step generation}
\label{sec:latent_transport}

\nameshort addresses this question by learning a continuous coupling between
discrete sequences and Gaussian latents. The final decoder still emits token
distributions in parallel, but it conditions all positions on a shared latent
$z$. Hence, although $G_\theta(x\mid z)=\prod_t G_\theta(x_t\mid z)$ factorizes
conditioned on $z$, the marginal
$p_\theta^{\mathrm{gen}}(x)=\int G_\theta(x\mid z)p_Z(z)\,dz$ is a latent mixture
that can carry global dependence.

The method has two stages. Stage~A maps each training sequence to a continuous
representation and uses a likelihood-based normalizing flow to align the induced
latent marginal with $p_Z=\mathcal{N}(0,I)$. This yields paired supervision
$(z,x)$, where $z$ is informative about $x$ and sampleable from the inference
prior. Stage~B trains a parallel decoder to invert these pairs, so inference
reduces to drawing $z\sim\mathcal{N}(0,I)$ and decoding all tokens in one
forward pass. Unlike flow-map compression, \nameshort does not learn an endpoint
map along a discrete or simplex-valued trajectory; it first constructs a
sampleable Gaussian coupling and then learns the one-step inverse map.
Appendix~\ref{app:trilemma} formalizes the expressivity barrier for direct
token-wise one-step decoders.

\xhdr{Stage A: Discrete sequence $\to$ Gaussian}
Stage~A constructs the latent--sequence coupling used by the one-step decoder.
Since the sequence $x$ is discrete, we do not fit a continuous density model
directly on tokens. Instead, we first encode $x$ into a continuous representation
$u$, and then transform this representation into a Gaussian latent $z$ using a
normalizing flow. This separates two roles: the encoder space should preserve
the information needed to reconstruct the sequence, while the flow space should
match the simple prior used for sampling.

Concretely, we parameterize a reparameterized encoder
$q_\psi(u\mid x)$ by
$
u = E_\psi(x,\epsilon), \epsilon\sim\mathcal{N}(0,I),
$
and map the resulting representation to the sampling latent through an
invertible normalizing flow,
$z=\mathrm{NF}_\phi(u),  p_Z=\mathcal{N}(0,I).$
A lightweight decoder $D$ reconstructs the original sequence from $u$. The
Stage~A objective is
\begin{equation}
\begin{aligned}
\mathcal{L}_{A}
&=
\lambda_{\mathrm{rec}}
\underbrace{
\left[-\sum_{t=1}^{T}\log p_D(x_t\mid u)\right]
}_{\mathcal{L}_{\mathrm{rec}}}
+
\lambda_{\mathrm{KL}}
\underbrace{
\mathrm{KL}\!\left(q_\psi(u\mid x)\,\|\,\mathcal{N}(0,I)\right)
}_{\mathcal{L}_{\mathrm{KL}}} \\
&\quad+
\lambda_{\mathrm{flow}}
\underbrace{
\left[
-\log p_Z(\mathrm{NF}_\phi(u))
-\log\left|\det
\frac{\partial \mathrm{NF}_\phi(u)}{\partial u}
\right|
\right]
}_{\mathcal{L}_{\mathrm{flow}}}.
\end{aligned}
\label{eq:stage_a}
\end{equation}

Each term enforces a distinct property of the coupling. The reconstruction loss
prevents the continuous representation from discarding sequence-level
information. The KL term regularizes the encoder distribution and stabilizes the
continuous representation space. The flow likelihood then aligns the induced
latent marginal with the Gaussian prior used at inference. Thus, Stage~A is not
used as an exact likelihood model for the discrete sequence itself. Its purpose
is to produce paired supervision
\[
(z,x), \qquad z=\mathrm{NF}_\phi(E_\psi(x,\epsilon)),
\]
where $z$ is both informative about $x$ and approximately distributed as the
sampling prior. These pairs are the training data for Stage~B.

\xhdr{Stage B: Gaussian $\to$ Discrete sequence}
After Stage~A, we freeze $E_\psi$, $D$, and $\mathrm{NF}_\phi$ and construct
paired data $\mathcal{D}_Z=\{(z^{(i)},x^{(i)})\}_{i=1}^{N}$ with
$z^{(i)}=\mathrm{NF}_\phi(E_\psi(x^{(i)},\epsilon^{(i)}))$. We train a
parallel decoder $G_\theta$ on these pairs:
\begin{equation}
\begin{aligned}
\ell &= G_\theta(z)\in\mathbb{R}^{T\times V}, \qquad
G_\theta(x\mid z)=\prod_{t=1}^{T}\mathrm{softmax}(\ell_t)[x_t], \\
\mathcal{L}_{B}(\theta)
&= -\mathbb{E}_{(z,x)\sim\mathcal{D}_Z}\log G_\theta(x\mid z).
\end{aligned}
\label{eq:stage_b}
\end{equation}
Although the categorical emission factorizes across positions given $z$, the
generated marginal does not reduce to an unconditional product distribution:
$p_\theta^{\mathrm{gen}}(x)=\int G_\theta(x\mid z)p_Z(z)\,dz$ is a latent
mixture of parallel decoders. The global dependence is therefore carried by the
coupling-induced latent variable, while the decoder performs the final
conditional token emission.
At inference time, we do not run Stage~A: we draw $z\sim\mathcal{N}(0,I)$, compute
$\ell=G_\theta(z)$, and sample
$\hat{x}_t\sim\mathrm{Cat}(\mathrm{softmax}(\ell_t/\tau))$ independently across
positions. Training proceeds in two passes. We first optimize Eq.~\eqref{eq:stage_a}
on the training set, freeze the encoder and flow, materialize $\mathcal{D}_Z$,
and then optimize Eq.~\eqref{eq:stage_b}. The full training and sampling
procedure is given in Appendix~\ref{app:algorithm}.

\xhdr{Why latent matching matters for one-step decoding}
A VAE-style encoder already produces latent--sequence pairs, so one may ask why the flow component is necessary. The issue is that Stage~B is trained on the latent marginal induced by Stage~A, but generation samples from the Gaussian prior. If these two latent distributions differ, the one-step decoder is used off distribution at inference. The flow therefore plays a specific role: it aligns the Stage~B training latents with the sampling prior.
This train--test alignment is the main reason for fitting a flow after the
continuous encoder: the decoder is trained on the same latent marginal from
which it will be sampled at inference.

\begin{propositionbox}{Latent matching controls one-step generation error}{latent_matching}
Let $q_\phi(z,x)$ be the latent--sequence joint induced by Stage~A, and let
$p_Z=\mathcal{N}(0,I)$ be the inference prior. For any decoder $G_\theta(x\mid z)$,
the generated distribution
$p_\theta^{\mathrm{gen}}(x)=\int G_\theta(x\mid z)p_Z(z)\,dz$ satisfies
\begin{equation}
\mathrm{TV}\!\left(
p_{\mathrm{data}},
p_\theta^{\mathrm{gen}}
\right)
\le
\mathbb{E}_{z\sim q_\phi(z)}
\mathrm{TV}\!\left(q_\phi(\cdot\mid z),G_\theta(\cdot\mid z)\right)
+
\mathrm{TV}\!\left(q_\phi(z),p_Z(z)\right)
\label{eq:latent_matching_tv}
\end{equation}
Moreover, if
\[
\mathbb{E}_{z\sim q_\phi(z)}
\mathrm{KL}\!\left(q_\phi(\cdot\mid z)\,\|\,G_\theta(\cdot\mid z)\right)
\le \varepsilon_{\mathrm{dec}},
\qquad
\mathrm{KL}\!\left(q_\phi(z)\,\|\,p_Z(z)\right)
\le \varepsilon_{\mathrm{flow}},
\]
then
\begin{equation}
\mathrm{TV}\!\left(
p_{\mathrm{data}},
p_\theta^{\mathrm{gen}}
\right)
\le
\sqrt{\frac{\varepsilon_{\mathrm{dec}}}{2}}
+\sqrt{\frac{\varepsilon_{\mathrm{flow}}}{2}} .
\label{eq:latent_matching_kl_tv}
\end{equation}
Consequently, if Stage~B matches $q_\phi(x\mid z)$ and Stage~A matches $q_\phi(z)$ to $p_Z$, then one-step sampling recovers the data distribution.
\end{propositionbox}

Proposition~\ref{prop:latent_matching} identifies the two quantities that the
two-stage procedure must control. The first term is the decoding error optimized
by Eq.~\eqref{eq:stage_b}; the second is the latent mismatch optimized by the
Stage~A flow objective. Thus, the flow is not introduced merely as a richer VAE
prior. It makes the latent--sequence coupling sampleable by aligning the latents
used for supervised decoder training with the Gaussian latents used at inference.
Appendix~\ref{app:latent_matching_proof} gives the full formal statement,
including the KL-based bound and proof.

\subsection{Few-step extension}
\label{sec:few_step_lcmdm}

The coupling construction is not restricted to a direct one-step decoder. Since
Stage~A produces paired latent--sequence supervision
$\mathcal{D}_Z=\{(z^{(i)},x^{(i)})\}_{i=1}^{N}$, Stage~B can be instantiated by
any conditional discrete generator. We therefore also consider a few-step
variant that replaces the one-step decoder $G_\theta$ with a latent-conditioned
masked denoiser $H_\theta$. This gives a direct extension from one-step decoding
to masked diffusion sampling while preserving the same learned Gaussian latent
space.

During training, $H_\theta$ receives a partially masked sequence together
with its Stage~A latent and predicts the clean tokens at the masked positions.
The latent therefore supplies global sequence-level information, while the
masked state provides the local refinement context that a direct one-step
decoder lacks.
For each pair $(z,x)\sim\mathcal{D}_Z$, we sample a masking time
$t\sim\mathrm{Unif}(0,1)$ and construct a partially masked sequence $x_t$ by
masking each position independently with probability $t$. Specifically, we
sample $m_i\sim\mathrm{Bernoulli}(t)$ for each position $i$, set
$(x_t)_i=\mathsf{M}$ if $m_i=1$ and $(x_t)_i=x_i$ otherwise, where
$\mathsf{M}$ is the mask token. The denoiser $H_\theta(x_t,z)$ is conditioned
only on the masked sequence $x_t$ and the Gaussian latent $z$, and predicts
logits over the vocabulary at all positions. We train it with the masked-token
cross-entropy objective
\begin{equation}
  \mathcal{L}_{\mathrm{MDM}}(\theta)
  =
  -\mathbb{E}_{(z,x)\sim\mathcal{D}_Z,\;t\sim\mathrm{Unif}(0,1),\;m}
  \left[
  \frac{1}{\sum_{i=1}^{T}m_i}
  \sum_{i:m_i=1}
  \log p_\theta(x_i\mid x_t,z)
  \right].
  \label{eq:few_step_lcmdm_loss}
\end{equation}
The training procedure is summarized in Algorithm~\ref{alg:lcmdm_training}.
At inference time, we sample one latent $z\sim\mathcal{N}(0,I)$,
keep it fixed throughout the trajectory, and run P2-self masked sampling
\citep{peng2025pathplanning}. With one reveal step, this reduces to a
latent-conditioned parallel decode; with more steps, the sampler can revise
low-confidence positions while conditioning on the same global latent. The full
masked-token objective and sampling algorithm are given in
Appendix~\ref{app:few_step_lcmdm}.

\subsection{Guided one-step generation}
\label{sec:guidance}

Guidance is a central tool in modern generative modeling, as it enables
controllable generation toward a target condition or property
\citep{dhariwal2021diffusion,ho2022classifierfree}. In continuous models, such
guidance can often be applied directly through the model output. In discrete
generation, however, the situation is more difficult because sampling is
non-differentiable. This difficulty becomes especially pronounced in current
multi-step discrete generators, where guidance or reward signals must be
propagated through an entire trajectory of stochastic transitions
\citep{sahoo2024mdlm,schiff2025simpleguidance,nisonoff2025unlocking,wang2025drakes}.
One-step discrete generators avoid this issue by acting on the final predictive
distribution, rather than being propagated through multiple discrete updates.
This shorter control path is empirically effective:
Section~\ref{sec:exp_guidance} shows that one-step guidance improves the
quality--efficiency tradeoff relative to the multi-step MDM baseline, and
Figure~\ref{fig:guidance_summary} summarizes both the guided comparison and the
CFG sweep. We use three standard mechanisms in this one-step setting.

\looseness=-1

\paragraph{Classifier-free guidance.}
Following classifier-free guidance~\citep{ho2022classifierfree}, we train
$G_\theta(z,y)$ with conditioning dropout and combine conditional and
unconditional logits as
\begin{equation}
\ell^{\mathrm{cfg}}=\ell^{u}+s(\ell^{c}-\ell^{u}),
\qquad
\ell^{c}=G_\theta(z,y),
\qquad
\ell^{u}=G_\theta(z,\varnothing),
\label{eq:cfg_guidance}
\end{equation}
which biases the final token distributions toward the requested condition
without changing the latent sample.
\paragraph{Classifier guidance.}
Following standard classifier-guidance methods
\citep{dhariwal2021diffusion,schiff2025simpleguidance,nisonoff2025unlocking},
we keep $G_\theta$ fixed and optimize $z$ by gradient ascent on a differentiable
relaxation of the decoded sequence before the single final sample:
\begin{equation}
z^{k+1}=z^k+\eta\nabla_{z^k}
R\!\left(\rho(G_\theta(z^k,y)),y\right),
\label{eq:latent_classifier_guidance}
\end{equation}
which changes the latent draw rather than the decoder parameters.
\paragraph{Reward fine-tuning.}
For reward fine-tuning~\citep{wang2025drakes}, we update the generator with a reward term
and an anchor penalty to keep it close to the pretrained model:
\begin{equation}
\mathcal{L}_{\mathrm{FT}}(\theta)
=
-\lambda_{\mathrm{reward}}
\mathbb{E}_{z,y}
R\!\left(\rho(G_\theta(z,y)),y\right)
+\lambda_{\mathrm{anchor}}\mathcal{R}(\theta,\theta_0).
\label{eq:reward_finetuning}
\end{equation}
These mechanisms act at different places in the one-step pipeline: logits for
CFG, the latent variable for classifier guidance, and model parameters for
reward fine-tuning. In all cases, the relaxation is needed only at the final
prediction layer, rather than through a denoising trajectory.

\paragraph{Differentiable guidance through final-layer relaxations.}
Classifier guidance and reward fine-tuning require gradients through the generated
sequence, but discrete sampling is not differentiable. We use two simple surrogate paths.
The first replaces the sampled sequence by a soft relaxation of the final logits,
using a softmax relaxation for categorical outputs to give stable low-variance gradients. The second uses a
Gumbel-Softmax or straight-through relaxation, which more closely follows the
categorical sampling path at the cost of noisier gradients. In both cases, the
relaxation is used only to compute guidance or fine-tuning gradients; the final
sample is still drawn discretely from the guided logits. Full procedural details
are given in Appendix~\ref{app:guidance_relaxations}.

\section{Experiments}

\begin{table}[t]
\centering
\caption{
\textbf{Unconditional generation on MNIST-Binary.} We report FID at a specified sampling budget (Steps).
}
\label{tab:mnist_binary_main_1col}

\footnotesize
\renewcommand{\arraystretch}{1.06}
\setlength{\tabcolsep}{2.3pt}

\sisetup{
  detect-weight=true,
  detect-inline-weight=math,
  table-number-alignment=center,
  table-text-alignment=center
}

\newcommand{\faded}{\color{black!55}}

\newcolumntype{L}[1]{>{\raggedright\arraybackslash}p{#1}}

\begin{minipage}[t]{0.37\columnwidth}
\vspace{0pt}\centering
\begin{tabular}{@{}L{0.67\linewidth} c S[table-format=3.2]@{}}
\toprule
\multicolumn{3}{@{}l}{\textbf{One-step}} \\
\midrule
Method & Steps & {FID$\downarrow$} \\
\midrule
\MDLM & 1 & 204.64 \\
\UDLM & 1 & 130.57 \\
\PAIRFLOW & 1 & 40.58 \\
\UDLM{} + \DCD & 1 & 53.84 \\
\PAIRFLOW{} + \DCD & 1 & 19.51 \\
\UDLM{} + \ReDi & 1 & 18.44 \\
\PAIRFLOW{} + \ReDi & 1 & 12.90 \\
\CFM & 1 & 10.10 \\
\midrule
\rowcolor{black!6}
\nameshort (Ours) & 1 & \textbf{5.50} \\
\bottomrule
\end{tabular}
\end{minipage}
\hfill
\begin{minipage}[t]{0.275\columnwidth}
\vspace{0pt}\centering
\faded
\begin{tabular}{@{}L{0.56\linewidth} c S[table-format=3.2]@{}}
\toprule
\multicolumn{3}{@{}l}{\textbf{Few-step}} \\
\midrule
Method & Steps & {FID$\downarrow$} \\
\midrule
\MDLM & 2 & 159.26 \\
\MDLM & 4 & 103.74 \\
\addlinespace[1pt]
\UDLM & 2 & 42.54 \\
\UDLM & 4 & 11.25 \\
\addlinespace[1pt]
\PAIRFLOW & 2 & 15.61 \\
\PAIRFLOW & 4 & 8.50 \\
\addlinespace[1pt]
\CFM & 2 & 8.70 \\
\CFM & 4 & 7.80 \\
\bottomrule
\end{tabular}
\end{minipage}
\hfill
\begin{minipage}[t]{0.275\columnwidth}
\vspace{0pt}\centering
\faded
\begin{tabular}{@{}L{0.56\linewidth} c S[table-format=3.2]@{}}
\toprule
\multicolumn{3}{@{}l}{\textbf{Multi-step}} \\
\midrule
Method & Steps & {FID$\downarrow$} \\
\midrule
\SFMOT & \multicolumn{1}{c}{300} & 4.62 \\
\DirichletFM & \multicolumn{1}{c}{--} & 77.35 \\
\DDSM & \multicolumn{1}{c}{100} & 7.79 \\
\DThreePM & \multicolumn{1}{c}{1000} & 67.36 \\
\DFM & \multicolumn{1}{c}{--} & 34.42 \\
\addlinespace[2pt]
\MDLM & 64 & 7.01 \\
\UDLM & 64 & 5.01 \\
\PAIRFLOW & 64 & 5.17 \\
\bottomrule
\end{tabular}
\end{minipage}

\vspace{-0.8ex}
\end{table}

\label{sec:experiments}
We evaluate whether \nameshort provides a practical route to one-step
discrete generation across three established benchmarks: MNIST-Binary
unconditional image generation, Fly Brain enhancer generation, and LM1B
unconditional text generation. These tasks span controlled binary images,
structured biological sequences, and open-ended language, allowing us to test
whether the same coupling-based construction remains useful across qualitatively
different discrete domains.
 The main paper keeps the benchmark setup concise but explicit, while
Appendix~\ref{app:exp_details} collects the full implementation details. In
particular, Appendix~\ref{app:exp_mnist}, Appendix~\ref{app:exp_enhancer}, and
Appendix~\ref{app:exp_lm1b} provide the complete training and evaluation
protocols for the three tasks discussed below.
\subsection{\namelong improves one-step generation across domains}
\label{sec:one_step_results}

We first evaluate the strict one-step setting, where generation consists of one
Gaussian draw followed by one parallel decode. Tables~\ref{tab:mnist_binary_main_1col}
and~\ref{tab:dna_enhancer_clean} report the image and DNA results, while
Figure~\ref{fig:lm1b-entropy-genppl-tradeoff} reports the LM1B
entropy--perplexity tradeoff. Across all three settings, \nameshort
improves the available one-step baselines under the corresponding evaluation
protocol.

\xhdr{MNIST-Binary setup}
MNIST-Binary is a controlled discrete image-generation benchmark in which each
binarized $28\times28$ image is treated as a discrete grid. We follow the
unconditional-generation protocol used in \CFM{} and report FID as the primary sample-quality metric. Our model uses the
two-stage \nameshort pipeline: Stage~A learns a latent coupling on a $7\times7$
latent grid with 16 channels, and Stage~B trains a one-step binary decoder. The
full benchmark configuration is reported in Appendix~\ref{app:exp_mnist}.

Table~\ref{tab:mnist_binary_main_1col} shows the strongest result for the
\nameshort model. With one decoding step, it achieves FID $5.50$, improving over
all one-step baselines in the comparison, including \CFM at $10.10$ and
\PAIRFLOW{}+\ReDi at $12.90$. The result is also competitive with several
few-step and multi-step baselines, while preserving true one-step sampling.
Figure~\ref{fig:uncond_mnist_results} shows random samples from the trained
generator.

\begin{wraptable}[16]{r}{0.5\linewidth}
\vspace{-1.1\baselineskip}
\centering
\caption{\textbf{Fly Brain enhancer generation.}}
\label{tab:dna_enhancer_clean}
\scriptsize
\setlength{\tabcolsep}{2.6pt}
\renewcommand{\arraystretch}{1.04}

\begin{tabular}{@{}L{0.61\linewidth}cc@{}}
\toprule
Model & Steps & FB FBD $\downarrow$ \\
\midrule

\multicolumn{3}{@{}l}{\textbf{Reference}}\\
\midrule
Random Sequence & -- & 876 \\
\midrule
\multicolumn{3}{@{}l}{\textbf{Multi-step}}\\
\midrule
Linear FM~\citep{stark2024dirichletfm} & 100 & 15.0 \\
Dirichlet FM~\citep{stark2024dirichletfm} & 100 & 15.2 \\
\SLM & 100 & 4.4 $\pm$ 0.2 \\
\FisherFlow & 100 & 3.8 $\pm$ 0.3 \\
Autoregressive~\citep{stark2024dirichletfm} & 500 & 25.2 \\
\midrule
\multicolumn{3}{@{}l}{\textbf{One-step}}\\
\midrule
Dirichlet FM (dist.)~\citep{stark2024dirichletfm} & 1 & 15.8 \\
\rowcolor{black!6}
\textbf{\nameshort (Ours)} & \textbf{1} & \textbf{12.9 $\pm$ 0.3} \\
\bottomrule
\end{tabular}
\vspace{-0.9\baselineskip}
\end{wraptable}
\textbf{Fly Brain enhancer setup.}
For DNA enhancer generation, we follow the Fly Brain benchmark from \DFM{}. Sequences have length 500 over $\{A,C,G,T\}$, and the benchmark
contains 81 target classes. We report FBD, the Fr\'echet distance in the frozen
classifier embedding space used by prior work. The corresponding training and
evaluation details are in Appendix~\ref{app:exp_enhancer}.

\looseness=-1
Table~\ref{tab:dna_enhancer_clean} shows that \nameshort also transfers
to structured biological sequences. It reduces Fly Brain FBD from $15.8$ for the
distilled one-step Dirichlet FM baseline to $12.9\pm0.3$. Stronger 100-step
simplex diffusion and flow methods, especially \SLM and \FisherFlow, still
perform better. We therefore view this result as evidence that \nameshort is meaningful beyond binary images, while leaving a clear gap to stronger
iterative generators in harder biological sequence settings.

\looseness=-1
\xhdr{LM1B setup}
For LM1B, we use Figure~\ref{fig:lm1b-entropy-genppl-tradeoff} as the main
one-step result, and Table~\ref{tab:lm1b-entropy-genppl-raw} lists the raw
numbers behind that tradeoff. We follow the FLM evaluation protocol and report
generative perplexity under a GPT-2 Large scorer together with unigram entropy
on 1{,}024 generated samples. Gen-PPL measures sample quality under an external
language model, while entropy separates diverse generation from low-entropy
degenerate samples. \nameshort uses a frozen Qwen2.5-0.5B encoder in
Stage~A and a one-step Transformer decoder in Stage~B; the full text-generation
setup is given in~\S\ref{app:exp_lm1b}.

\looseness=-1
At the lowest temperature in the sweep, \nameshort reaches
Gen-PPL~45.82 with entropy~3.6088, which is slightly better in perplexity than
the LM1B reference point at Gen-PPL~47.47, though with lower entropy than the
reference entropy of 4.31. As the temperature increases to 0.60 and 0.65, entropy
rises to 4.0841 and 4.3192, while Gen-PPL remains at 76.19 and 98.96,
respectively. Compared with one-step baselines such as \FMLM{} at 119.34 Gen-PPL
and 4.16 entropy, or \DFM{} (ESD) at 68.11 Gen-PPL and 3.79 entropy, the sweep
shows a cleaner tradeoff between quality and diversity than the prior one-step
points. This result is weaker than the MNIST-Binary case in absolute quality,
but it shows that the same \nameshort construction remains credible for
open-ended text generation.

\subsection{Few-step refinement}
\label{sec:lm1b_step_efficiency}

\begin{wrapfigure}[15]{r}{0.5\textwidth}
    \centering
    \includegraphics[width=0.95\linewidth]{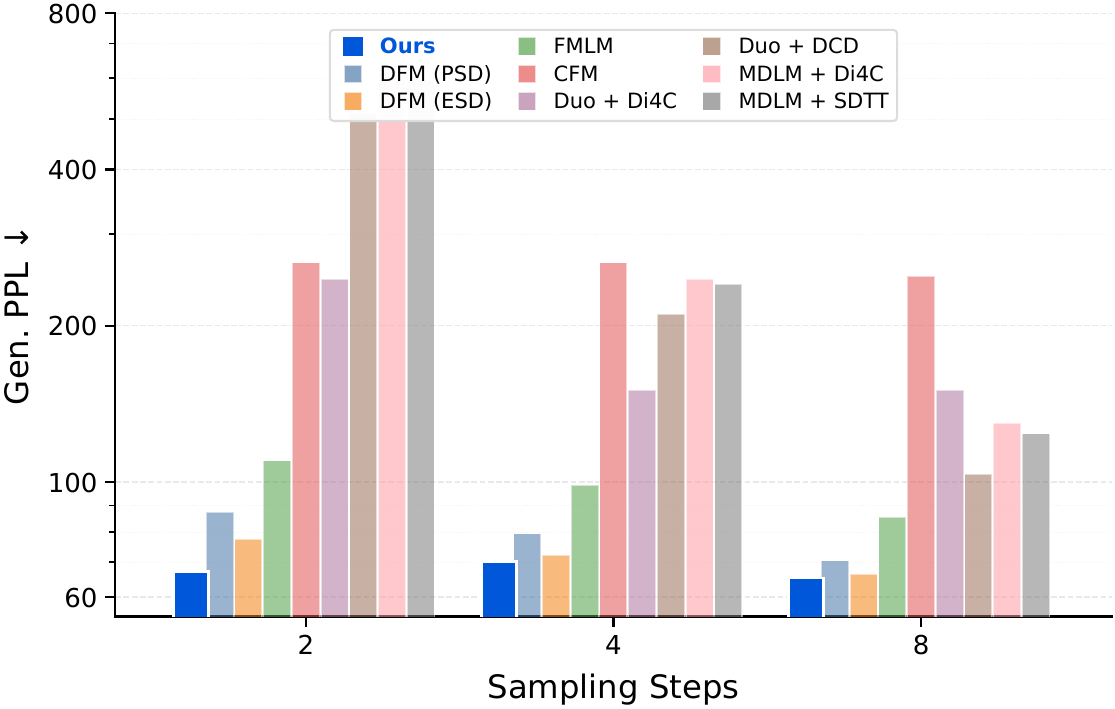}
    \caption{
\textbf{\nameshort remains competitive for few-step LM1B generation.}
}
    \label{fig:genppl_vs_sampling_steps}
\end{wrapfigure}

The one-step results test whether \nameshort can support direct
parallel decoding. We next ask whether the same learned Gaussian latent space
remains useful when the generator is allowed a small number of masked-refinement
steps.
Figure~\ref{fig:genppl_vs_sampling_steps} compares LM1B Gen-PPL at 2, 4, and
8 sampling steps under the same evaluation protocol. \nameshort draws
one Gaussian latent $z\sim\mathcal{N}(0,I)$ and keeps it fixed throughout the
masked-refinement trajectory. The additional steps refine
local token decisions, while global sequence-level information is supplied
by the continuous Gaussian latent.
\nameshort achieves the best Gen-PPL among the compared low-step
methods at all three budgets. At 2 steps, it reaches Gen-PPL $66.2$, improving
over the strongest baseline DFM-ESD at $76.6$ and DFM-PSD at $86.5$. At 4 steps,
it obtains $69.3$, still slightly better than DFM-ESD at $71.2$ and substantially
better than DFM-PSD at $78.5$ and FMLM at $97.6$. At 8 steps, it reaches its best
score, $64.3$, improving over DFM-ESD at $65.4$, DFM-PSD at $69.5$, and FMLM at
$84.4$. Other low-step acceleration methods, including CFM, Duo-based
distillation, and MDLM-based distillation, remain substantially worse in this
low-step regime.

\subsection{Efficient one-step guidance}
\label{sec:exp_guidance}

Finally, we test the guidance claim suggested by the method. In multi-step
discrete generators, guidance must influence a long trajectory of
non-differentiable stochastic token updates. The one-step model instead acts on
the final logits, avoiding trajectory-level credit assignment. We compare
against a larger $32\times32$ U-Net masked diffusion model trained on the same
MNIST-Binary task with the same data budget. All evaluations use balanced digit
labels, 1{,}000 generated samples, the same reward model, and the same held-out
evaluation classifier. Guidance algorithms and hyperparameters are given in
Appendix~\ref{app:guidance}.

Figure~\ref{fig:guidance_summary} shows that one-step guidance is more
efficient across all three mechanisms while matching or improving FID. With
classifier-free guidance, \nameshort obtains FID $5.64$ and $99.1\%$ class
accuracy using $2$ function evaluations, compared with FID $8.72$ and $99.8\%$
accuracy for MDM using $512$. With classifier guidance, \nameshort reaches FID
$5.26$ and $98.5\%$ accuracy using $6$ evaluations, while MDM obtains FID
$7.95$ and $97.2\%$ accuracy using $256$. Reward fine-tuning gives the strongest
\nameshort result: FID $5.09$, $99.5\%$ accuracy, and $0.63$ seconds per
1{,}000 samples, compared with FID $8.82$, $99.7\%$, and $203.04$ seconds for
reward-fine-tuned MDM.

\begin{figure}[!t]
\centering
\vspace{-0.5\baselineskip}
\setlength{\tabcolsep}{0pt}
\begin{tabular}{@{}c@{\hspace{0.03\linewidth}}c@{}}
\begin{minipage}[t]{0.64\linewidth}
\centering
\vspace{0pt}
\scriptsize
\setlength{\tabcolsep}{2.1pt}
\renewcommand{\arraystretch}{1.0}
\resizebox{\linewidth}{!}{%
\begin{tabular}{@{}llccccc@{}}
\toprule
Model & Guidance & Params (M) & FID $\downarrow$ & Acc. $\uparrow$ & NFE $\downarrow$ & Sec/1k $\downarrow$ \\
\midrule
MDM  & CFG        & 73.74 & 8.72 & \textbf{99.8} & 512 & 356.82 \\
MDM  & Classifier & 73.74 & 7.95 & 97.2          & 256 & 244.68 \\
MDM  & Reward FT  & 73.74 & 8.82 & 99.7          & 256 & 203.04 \\
\midrule
Ours & CFG        & \textbf{24.07} & 5.64          & 99.1 & 2          & 1.72 \\
Ours & Classifier & \textbf{24.07} & 5.26          & 98.5 & 6          & 13.20 \\
Ours & Reward FT  & \textbf{24.07} & \textbf{5.09} & 99.5 & \textbf{1} & \textbf{0.63} \\
\bottomrule
\end{tabular}
}
\end{minipage}
&
\begin{minipage}[t]{0.29\linewidth}
\centering
\vspace{0pt}
\includegraphics[width=0.95\linewidth]{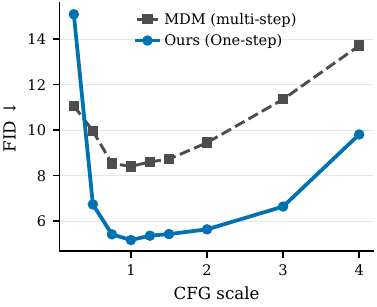}
\end{minipage}
\end{tabular}
\vspace{0.2\baselineskip}
\caption{
\textbf{Guidance is efficient and effective in one-step generation.}
\textbf{Left}: \nameshort achieves lower FID than the MDM baseline for CFG,
classifier guidance, and reward fine-tuning while using far fewer function
evaluations. \textbf{Right}: \nameshort shows better quality--guidance
tradeoff across scales.
}
\label{fig:guidance_summary}
\end{figure}

\section{Conclusion}
We presented \namelongs as a representation-based route to fast discrete
generation. Rather than compressing an existing denoising or flow trajectory,
the method learns a sampleable coupling between discrete sequences and Gaussian
latents, then trains a parallel decoder on the induced latent--sequence pairs.
At inference time, generation is one-step: sample $z\sim\mathcal{N}(0,I)$ and
decode once. Empirically, \nameshort gives a consistent low-step advantage
across MNIST-Binary, DNA enhancer design, and LM1B. The same latent interface
also supports direct guidance and few-step masked refinement, suggesting that
one-step discrete generation need not be viewed only as the limit of trajectory
compression.

\xhdr{Limitations}
The main limitation is scale. Although \nameshort is evaluated across images,
DNA, and language generation, the experiments use moderate-sized models and do
not test frontier-scale language generation, long-context sequence modeling, or
large-scale protein design. The results should therefore be read as evidence for
a one-step modeling route rather than a final scaling study. Finally, the metrics for language generation, like generative perplexity, are often not well correlated with final sample quality, a phenomenon that is also commonly observed in image generative models with likelihood-based evaluation~\citep{theis2015note,nalisnick2018deep}.

\begin{ack}
We thank Zachary Bezemek for helpful discussions and Antonio Franca for assistance with the LM1B experiments. F.Z.P. and A.R.Z. are partially supported by NIH R01HL169347.
\end{ack}

\clearpage
\bibliographystyle{plainnat}
\bibliography{reference}

\newpage
\appendix
\startappendixcontents
\appendixcontents

\section{Extended Related Work}
\label{app:related_work}

This appendix expands the concise discussion in Section~\ref{sec:related_work}. The main paper groups prior work by inference regime: serial autoregressive generation, iterative diffusion and flow-based generation, and recent attempts to compress discrete generation into one or a few evaluations. Here we provide a more detailed account of those families and clarify how the proposed model construction differs from trajectory-compression and flow-map methods.

\subsection{Autoregressive sequence generation}
\label{app:rw_ar}

Autoregressive models represent a discrete sequence distribution through a product of conditional next-token distributions. This factorization has become the standard foundation for modern Transformer-based language models~\citep{vaswani2017attention,brown2020gpt3,touvron2023llama2,yang2024qwen2}. Its main advantage is expressivity: each token is generated while conditioning on all previously generated tokens, allowing the model to represent rich long-range dependencies through a sequence of local decisions. Its main limitation is inference structure. Generation proceeds token by token, so the number of model evaluations scales with sequence length. Non-autoregressive and insertion-based sequence models replace strict left-to-right decoding with parallel prediction, iterative refinement, or flexible insertion orders~\citep{gu2018nat,lee2018iterative,ghazvininejad2019maskpredict,stern2019insertion}. These models are an important precursor to the speed-quality tradeoff studied here, but they still rely on refinement or specialized decoding structure rather than a single latent-to-sequence map.

\subsection{Iterative discrete diffusion, masked refinement, and discrete flows}
\label{app:rw_iterative}

Discrete diffusion and masked-refinement models provide a more parallel alternative to autoregressive decoding. Instead of generating tokens sequentially, these models update many positions at once and recover global structure through repeated denoising, masking, or replacement steps~\citep{austin2021d3pm,hoogeboom2021argmax,hoogeboom2022ardm,chen2022analogbits,sahoo2024mdlm,lou2024sedd,zheng2024reparameterized,schiff2025discreteguidance}. Text diffusion models have developed both continuous-token and discrete-token variants, including embedding-space, simplex-based, latent, and non-autoregressive formulations~\citep{li2022diffusionlm,gong2023diffuseq,han2023ssdlm,lovelace2023latent,gong2023diffuseqv2,zhou2024diffusionnat,nie2025scalingmdm,gong2025diffullama,peng2025papl}. This improves parallelism across token positions, but the model must still be evaluated many times at inference. Discrete flow and simplex-based methods pursue related goals through transport or flow-matching objectives over categorical variables or their continuous relaxations~\citep{gat2024discrete,cheng2024sfm,stark2024dirichletfm,davis2024fisherflow}. Related masked-token generators in vision and vector-quantized image models also use iterative parallel decoding over discrete codebooks~\citep{gu2022vqdiffusion,chang2022maskgit,chang2023muse,hu2024mask}. In biological sequence modeling, diffusion language models have been adapted to protein generation and multimodal protein representations~\citep{wang2024dplm,wang2025dplm2}. These methods show that strong sample quality can be obtained when multiple denoising steps are allowed. They also sharpen the central question of this paper: whether one-step generation can recover useful global structure without relying on a long refinement trajectory.

\subsection{One-step and few-step acceleration}
\label{app:rw_onestep}

A growing line of work asks whether the inference cost of iterative discrete generators can be compressed into the one-step or few-step regime. Distillation and consistency methods train a compressed sampler or student model to imitate a stronger multi-step process, reducing the number of model evaluations required at test time~\citep{sahoo2025diffusionduality,deschenaux2025beyond,zhu2025dimo,chen2025dlmone}. Rectification methods improve the low-step trajectory itself by reducing the error induced when an iterative process is aggressively shortened~\citep{yoo2025redi,monsefi2025fsdfm,zheng2026didi,zhang2026t3d,bartosh2026fldd}. Coupling-based methods provide another route: \PAIRFLOW{} constructs source-target couplings for few-step discrete flow models and shows that the choice of coupling can strongly affect low-step performance~\citep{park2026pairflow}. These approaches demonstrate that fast discrete generation is feasible, but their central object is still a compressed, rectified, or otherwise improved generative trajectory.

\subsection{Flow-map methods}
\label{app:rw_flowmaps}

Flow-map methods are especially close to the setting studied in this paper. In continuous generative modeling, a flow map directly predicts a later or terminal state of a generative trajectory, thereby avoiding expensive numerical integration. Recent work adapts this idea to discrete data by operating over continuous relaxations, one-hot embeddings, or probability-simplex representations. \DiscFM{} formulates discrete flow maps with objectives aligned to simplex geometry~\citep{potaptchik2026discreteflowmaps}. \FlowMapLM{} learns continuous flows over one-hot token embeddings and distills them into flow maps for efficient language generation~\citep{lee2026flm}. \CFM{} defines categorical flow maps toward the simplex and trains them with self-distillation and endpoint-consistency objectives~\citep{roos2026categoricalflowmaps}. Although these methods differ in parameterization and loss design, they share a common premise: fast discrete generation can be obtained by learning better endpoint maps for continuous or simplex-valued trajectories.

\subsection{Position of our method}
\label{app:rw_position}

Our method targets the same practical regime as recent one-step and few-step generators, but it uses a different construction. We do not distill a multi-step teacher, rectify a discrete trajectory, or learn an endpoint map along a simplex-valued path. Instead, we first learn a coupling between discrete sequences and Gaussian latents. This coupling induces paired supervision between sampleable latents and target sequences. A parallel decoder is then trained as a supervised inverse map from Gaussian latents to discrete outputs. The distinction is therefore not only architectural but also supervisory: flow-map and distillation methods inherit supervision from a generative trajectory, whereas our decoder is trained on latent-sequence pairs created by a learned sequence--noise coupling.

\section{Training and sampling algorithm}
\label{app:algorithm}

\begin{algorithm}[H]
  \caption{Two-stage training and one-step sampling}
  \label{alg:transport_training}
  \small
  \begin{algorithmic}[1]
    \State \textbf{Input:} training set $\mathcal{D}=\{x^{(i)}\}_{i=1}^{N}$
    \For{Stage~A updates}
      \State sample a minibatch $x\sim\mathcal{D}$ and noise $\epsilon\sim\mathcal{N}(0,I)$
      \State $u \gets E_\psi(x,\epsilon)$
      \State $\hat{x}^{A} \gets D(u)$, \quad $z \gets \mathrm{NF}_\phi(u)$
      \State update $(\psi,\phi,D)$ by minimizing $\mathcal{L}_{A}$ in Eq.~\eqref{eq:stage_a}
    \EndFor
    \State freeze $E_\psi$, $D$, and $\mathrm{NF}_\phi$
    \State construct $\mathcal{D}_Z=\{(\mathrm{NF}_\phi(E_\psi(x^{(i)},\epsilon^{(i)})),x^{(i)})\}_{i=1}^{N}$
    \For{Stage~B updates}
      \State sample a minibatch $(z,x)\sim\mathcal{D}_Z$
      \State update $G_\theta$ by minimizing $\mathcal{L}_{B}$ in Eq.~\eqref{eq:stage_b}
    \EndFor
    \State \textbf{Sampling:} draw $z\sim\mathcal{N}(0,I)$, compute $\ell \gets G_\theta(z)$, and sample $\hat{x}_t \sim \mathrm{Cat}(\mathrm{softmax}(\ell_t/\tau))$
  \end{algorithmic}
\end{algorithm}

\section{Expressivity barrier of direct one-step token-wise decoders}
\label{app:trilemma}

This appendix expands the short motivation in Section~\ref{sec:method}. There, we argued that the difficulty of one-step discrete generation is not merely an optimization issue, but also a representational one: a direct one-step categorical decoder must recover global cross-token dependence in a single parallel prediction. Here we make that point precise, provide a complete proof of the expressivity barrier underlying Eq.~\eqref{eq:factorized_onestep_main}, give additional examples, and clarify how this perspective relates to existing model families and to our method.

\subsection{Why a direct one-step token-wise decoder is limited}
\label{app:trilemma_formal}

In the main text, the key observation was that a direct one-step discrete decoder typically predicts one categorical distribution per position and then samples all tokens independently, yielding a factorized sequence law
\begin{equation}
p(x)=\prod_{i=1}^{L} p_i(x_i).
\label{eq:factorized_onestep_main}
\end{equation}
We now formalize the family induced by this design.

Let $\mathcal{X} = [V]^L$ denote the discrete sequence space with vocabulary size $V$ and sequence length $L$, and let $\Delta(\mathcal{X})$ denote the simplex of all probability distributions over $\mathcal{X}$. Consider the family
\begin{equation}
\mathcal{F}_{\mathrm{fact}}
=
\left\{
q \in \Delta(\mathcal{X})
\;:\;
q(x)=\prod_{i=1}^{L} q_i(x_i), \;\; q_i \in \Delta([V])
\right\}.
\label{eq:factorized_family_appendix}
\end{equation}
This is exactly the distribution family induced by a direct one-step token-wise decoder with independent per-position sampling.

The main-text claim can now be stated precisely: the family in Eq.~\eqref{eq:factorized_family_appendix} is too small to represent a general joint law over sequences.

\begin{proposition}[Expressivity barrier of direct one-step token-wise decoders]
\label{prop:factorized_barrier_appendix}
For any $V,L \ge 2$, the factorized family $\mathcal{F}_{\mathrm{fact}}$ is a strict subset of $\Delta([V]^L)$. Equivalently, a direct one-step token-wise discrete decoder is not universally expressive over sequence distributions.
\end{proposition}

\xhdr{Proof}
We prove the claim in two complementary ways.

First, consider the degrees of freedom. The full sequence distribution lives in the simplex $\Delta([V]^L)$, whose dimension is
\begin{equation}
\dim \Delta([V]^L)=V^L-1.
\label{eq:full_simplex_dim_appendix}
\end{equation}
By contrast, each factor $q_i \in \Delta([V])$ contributes only $V-1$ free parameters, so the factorized family in Eq.~\eqref{eq:factorized_family_appendix} has dimension
\begin{equation}
\dim \mathcal{F}_{\mathrm{fact}}=L(V-1).
\label{eq:factorized_dim_appendix}
\end{equation}
For all $V,L \ge 2$,
\begin{equation}
L(V-1) < V^L-1.
\label{eq:dim_gap_appendix}
\end{equation}
Therefore, $\mathcal{F}_{\mathrm{fact}}$ cannot coincide with the full sequence simplex. In other words, the family induced by direct one-step token-wise decoding is a strict low-dimensional subset of all discrete sequence distributions.

Second, it is useful to exhibit a concrete distribution that lies outside this family. Consider the binary length-two case $V=2$, $L=2$, and define
\begin{equation}
p^\star(0,0)=\tfrac12,\qquad
p^\star(1,1)=\tfrac12,\qquad
p^\star(0,1)=p^\star(1,0)=0.
\label{eq:correlated_example_appendix}
\end{equation}
This distribution places all probability mass on the two perfectly correlated outcomes. Suppose, for contradiction, that $p^\star \in \mathcal{F}_{\mathrm{fact}}$. Then there exist $q_1,q_2 \in \Delta([2])$ such that
\begin{equation}
p^\star(x_1,x_2)=q_1(x_1)q_2(x_2).
\end{equation}
Any such factorized distribution must satisfy
\begin{equation}
p^\star(0,0)\,p^\star(1,1)=p^\star(0,1)\,p^\star(1,0),
\label{eq:product_constraint_appendix}
\end{equation}
because both sides are equal to $q_1(0)q_1(1)q_2(0)q_2(1)$. But for the target distribution in Eq.~\eqref{eq:correlated_example_appendix},
\begin{equation}
p^\star(0,0)\,p^\star(1,1)=\tfrac14,
\qquad
p^\star(0,1)\,p^\star(1,0)=0,
\end{equation}
which is impossible. Hence $p^\star \notin \mathcal{F}_{\mathrm{fact}}$.

Together, the dimension argument and the explicit counterexample establish the proposition.
\hfill $\square$

\subsection{Interpretation of the degree-of-freedom gap}
\label{app:trilemma_interpretation}

The main text summarized the above result by noting that a general sequence distribution has $V^L-1$ degrees of freedom, whereas the factorized family in Eq.~\eqref{eq:factorized_onestep_main} has only $L(V-1)$. This gap is the mathematical form of the representational barrier.

The key point is that the problem is structural rather than pathological. The family in Eq.~\eqref{eq:factorized_family_appendix} excludes any target law with nontrivial cross-token dependency. This includes simple correlated patterns such as Eq.~\eqref{eq:correlated_example_appendix}, but also more realistic structures such as repeated motifs, long-range constraints, parity-type dependencies, and other globally coordinated sequence patterns. Thus, before considering optimization, architecture, or training objective, a direct one-step token-wise discrete decoder is already too weak to model the kind of dependency structure that discrete sequence data typically exhibit.

This also explains why the factorized decoder is simultaneously attractive and limited. It is attractive because its output size grows only linearly with sequence length and vocabulary size. It is limited because linear scaling in $L$ cannot directly parameterize an exponentially large joint probability space.

\subsection{How existing model families resolve the trilemma}
\label{app:trilemma_model_families}

The main text then used this barrier to explain how existing model families resolve the trilemma.

Autoregressive models retain \emph{native discrete modeling} and expressive joint structure, but sacrifice \emph{parallel one-step generation}. Their key advantage is that they do not attempt to represent the full sequence law through a single factorized prediction of the form in Eq.~\eqref{eq:factorized_onestep_main}. Instead, they decompose the joint distribution into a sequence of conditional distributions, allowing dependency structure to accumulate through the generation order.

Masked diffusion models make a closely related tradeoff. They also remain in the native discrete space and typically use local token-wise predictions, but they do not rely on a single direct one-step decode. Instead, they recover expressive joint behavior through multiple denoising or refinement steps. In this sense, they preserve native discreteness and local parameterization, but give up the one-step requirement.

At the opposite extreme, exact one-step native discrete generation can be obtained by explicitly parameterizing the full joint distribution over the entire sequence space. However, this requires a number of outputs or equivalent representational complexity that scales as $O(V^L)$, which becomes computationally prohibitive for realistic sequence lengths.

Continuous-relaxation methods such as simplex flow and Fisher-Flow resolve the tension differently. Rather than directly parameterizing a discrete sequence law in one shot, they embed categorical data into a continuous surrogate geometry and perform transport in that continuous space. This preserves local parameterization and can support transport-style parallel generation, but it no longer remains within the class of direct native-discrete one-step decoders covered by Proposition~\ref{prop:factorized_barrier_appendix}.

\subsection{Connection to our design principle}
\label{app:trilemma_ours}

The main-text design principle now follows directly from the above discussion. The obstacle is not that a single forward pass is inherently incapable of producing a complex discrete sequence. Rather, the obstacle is that a \emph{direct} one-step token-wise categorical decoder is too weak to represent a general joint law over sequence space.

Our method therefore avoids assigning that burden to the final discrete decoder. Instead, we first learn a deterministic coupling between sequences and a continuous Gaussian latent space using a normalizing flow encoder trained by maximum likelihood. This coupling absorbs global dependency structure into a continuous representation. We then train a generator to decode from Gaussian noise to the target sequence in a single forward pass.

From the perspective of the trilemma, our method preserves one-step generation at inference time, but escapes the barrier in Proposition~\ref{prop:factorized_barrier_appendix} by giving up direct native-discrete one-shot parameterization and replacing it with a learned continuous surrogate space. This is precisely why the proposition does not contradict our modeling approach: our method is deliberately designed to fall outside the restrictive factorized family in Eq.~\eqref{eq:factorized_family_appendix}.

\subsection{Scope of the claim}
\label{app:trilemma_scope}

Finally, we clarify the scope of the result.

Proposition~\ref{prop:factorized_barrier_appendix} is not a universal impossibility theorem for all one-step generators. Rather, it is a precise structural barrier for a broad and practically important class of models: direct one-step token-wise discrete decoders with independent per-position sampling.

Accordingly, the proposition does \emph{not} imply that one-step discrete generation is impossible in full generality. A full-joint decoder over $[V]^L$ can represent any target law exactly, but is typically intractable because its complexity grows exponentially in sequence length. Likewise, the proposition does \emph{not} apply to methods that introduce a learned continuous surrogate space or coupling-based latent representation before decoding. Such methods step outside the class in Eq.~\eqref{eq:factorized_family_appendix}, which is exactly how they can preserve one-step generation without resorting to an explicit full-joint discrete parameterization.

For this reason, we view the trilemma not as a blanket impossibility statement for one-step generation, but as a useful lens for understanding why direct one-step native-discrete decoding is difficult, why iterative discrete models remain effective, and why a learned continuous coupling provides a natural alternative.

\section{Latent matching and one-step generation error}
\label{app:latent_matching_proof}

We give the formal version of Proposition~\ref{prop:latent_matching}. Let $p_{\mathrm{data}}(x)$ be the data distribution. Let $q_\psi(u\mid x)$ be the Stage~A encoder distribution and let $T_\phi=\mathrm{NF}_\phi$ denote the learned flow map. Define the Stage~A joint distribution $q_\phi(z,x)$ by
\[
x\sim p_{\mathrm{data}}, \qquad
u\sim q_\psi(u\mid x), \qquad
z=T_\phi(u).
\]
Let $q_\phi(z)$ and $q_\phi(x\mid z)$ denote the corresponding latent marginal and conditional sequence distribution. Let $p_Z=\mathcal{N}(0,I)$ be the prior used for one-step sampling, and define
\[
p_\theta^{\mathrm{gen}}(x)
=
\int G_\theta(x\mid z)p_Z(z)\,dz .
\]

\begin{proposition}[Latent matching controls one-step generation error]
\label{prop:latent_matching_full}
For any decoder $G_\theta(x\mid z)$,
\begin{equation}
\mathrm{TV}\!\left(
p_{\mathrm{data}},
p_\theta^{\mathrm{gen}}
\right)
\le
\mathbb{E}_{z\sim q_\phi(z)}
\mathrm{TV}\!\left(q_\phi(\cdot\mid z),G_\theta(\cdot\mid z)\right)
+
\mathrm{TV}\!\left(q_\phi(z),p_Z(z)\right).
\end{equation}
If
\[
\mathbb{E}_{z\sim q_\phi(z)}
\mathrm{KL}\!\left(q_\phi(\cdot\mid z)\,\|\,G_\theta(\cdot\mid z)\right)
\le
\varepsilon_{\mathrm{dec}}
\]
and
\[
\mathrm{KL}\!\left(q_\phi(z)\,\|\,p_Z(z)\right)
\le
\varepsilon_{\mathrm{flow}},
\]
then
\begin{equation}
\mathrm{TV}\!\left(
p_{\mathrm{data}},
p_\theta^{\mathrm{gen}}
\right)
\le
\sqrt{\frac{\varepsilon_{\mathrm{dec}}}{2}}
+
\sqrt{\frac{\varepsilon_{\mathrm{flow}}}{2}} .
\label{eq:latent_matching_final}
\end{equation}
In particular, if $q_\phi(z)=p_Z$ and $G_\theta(\cdot\mid z)=q_\phi(\cdot\mid z)$ for $q_\phi$-almost every $z$, then $p_\theta^{\mathrm{gen}}=p_{\mathrm{data}}$.
\end{proposition}

\begin{proof}
Let
\begin{equation}
p_\theta^{q}(x)
=
\int G_\theta(x\mid z)\,q_\phi(z)\,dz
\end{equation}
denote the distribution obtained by decoding from the Stage~A latent marginal rather than from the inference prior. By the triangle inequality,
\begin{equation}
\mathrm{TV}\!\left(p_{\mathrm{data}},p_\theta^{\mathrm{gen}}\right)
\le
\mathrm{TV}\!\left(p_{\mathrm{data}},p_\theta^{q}\right)
+
\mathrm{TV}\!\left(p_\theta^{q},p_\theta^{\mathrm{gen}}\right).
\end{equation}

For the first term, the Stage~A joint has marginal
\begin{equation}
p_{\mathrm{data}}(x)
=
\int q_\phi(x\mid z)\,q_\phi(z)\,dz .
\end{equation}
Therefore,
\begin{align}
\mathrm{TV}\!\left(p_{\mathrm{data}},p_\theta^{q}\right)
&=
\frac{1}{2}
\sum_x
\left|
\int
\left(
q_\phi(x\mid z)-G_\theta(x\mid z)
\right)
q_\phi(z)\,dz
\right| \\
&\le
\int
\frac{1}{2}
\sum_x
\left|
q_\phi(x\mid z)-G_\theta(x\mid z)
\right|
q_\phi(z)\,dz \\
&=
\mathbb{E}_{z\sim q_\phi(z)}
\mathrm{TV}\!\left(q_\phi(\cdot\mid z),G_\theta(\cdot\mid z)\right).
\end{align}

For the second term, $G_\theta(\cdot\mid z)$ is a Markov kernel from latents to sequences. Total variation cannot increase under a Markov kernel, so
\begin{equation}
\mathrm{TV}\!\left(p_\theta^{q},p_\theta^{\mathrm{gen}}\right)
\le
\mathrm{TV}\!\left(q_\phi(z),p_Z(z)\right).
\end{equation}
Combining these two bounds proves Eq.~\eqref{eq:latent_matching_tv}.

For the KL-based bound, Pinsker's inequality gives
\begin{equation}
\mathrm{TV}\!\left(q_\phi(\cdot\mid z),G_\theta(\cdot\mid z)\right)
\le
\sqrt{
\frac{1}{2}
\mathrm{KL}\!\left(q_\phi(\cdot\mid z)\,\|\,G_\theta(\cdot\mid z)\right)
}.
\end{equation}
Taking expectation over $z\sim q_\phi(z)$ and applying Jensen's inequality,
\begin{align}
\mathbb{E}_{z\sim q_\phi(z)}
\mathrm{TV}\!\left(q_\phi(\cdot\mid z),G_\theta(\cdot\mid z)\right)
&\le
\mathbb{E}_{z\sim q_\phi(z)}
\sqrt{
\frac{1}{2}
\mathrm{KL}\!\left(q_\phi(\cdot\mid z)\,\|\,G_\theta(\cdot\mid z)\right)
} \\
&\le
\sqrt{
\frac{1}{2}
\mathbb{E}_{z\sim q_\phi(z)}
\mathrm{KL}\!\left(q_\phi(\cdot\mid z)\,\|\,G_\theta(\cdot\mid z)\right)
} \\
&\le
\sqrt{\frac{\varepsilon_{\mathrm{dec}}}{2}} .
\end{align}
Finally, applying Pinsker's inequality to the latent marginal gives
\begin{equation}
\mathrm{TV}\!\left(q_\phi(z),p_Z(z)\right)
\le
\sqrt{
\frac{1}{2}
\mathrm{KL}\!\left(q_\phi(z)\,\|\,p_Z(z)\right)
}
\le
\sqrt{\frac{\varepsilon_{\mathrm{flow}}}{2}},
\end{equation}
which proves Eq.~\eqref{eq:latent_matching_final}. The final consistency statement follows by setting both terms in Eq.~\eqref{eq:latent_matching_tv} to zero.
\end{proof}

\section{Few-step latent-conditioned masked diffusion}
\label{app:few_step_lcmdm}

This appendix gives the full training and sampling procedures for the few-step
extension introduced in Section~\ref{sec:few_step_lcmdm}. The extension uses the
same Stage~A latent--sequence pairs as the one-step model. After Stage~A is
trained and frozen, each training sequence is mapped to a Gaussian latent
$z^{(i)}=\mathrm{NF}_\phi(E_\psi(x^{(i)},\epsilon^{(i)}))$ with
$\epsilon^{(i)}\sim\mathcal{N}(0,I)$, which yields the paired dataset
$\mathcal{D}_Z=\{(z^{(i)},x^{(i)})\}_{i=1}^{N}$. The one-step model trains a
direct decoder $G_\theta(z)$ on this paired dataset. The few-step variant
instead trains a latent-conditioned masked denoiser $H_\theta(x_t,z)$. Thus,
Stage~A is unchanged; only the Stage~B generator is replaced by a masked
diffusion model conditioned on the Gaussian latent.

\subsection{Training}
\label{app:few_step_lcmdm_training}

For each paired example $(z,x)\sim\mathcal{D}_Z$, we sample a masking time
$t\sim\mathrm{Unif}(0,1)$ and independently mask each position with probability
$t$. Let $m\in\{0,1\}^{T}$ be the resulting binary mask, where $m_i=1$ indicates
that position $i$ is masked. The corrupted sequence $x_t$ is obtained by setting
$(x_t)_i=\mathsf{M}$ if $m_i=1$ and $(x_t)_i=x_i$ otherwise. The denoiser
receives only the masked sequence and the latent variable, $H_\theta(x_t,z)$,
and predicts the clean sequence distribution. The loss is evaluated on the
masked positions:
\begin{equation}
  \mathcal{L}_{\mathrm{MDM}}(\theta)
  =
  -\mathbb{E}_{(z,x)\sim\mathcal{D}_Z,\;t\sim\mathrm{Unif}(0,1),\;m}
  \left[
  \frac{1}{\sum_{i=1}^{T}m_i}
  \sum_{i:m_i=1}
  \log p_\theta(x_i\mid x_t,z)
  \right].
  \label{eq:appendix_few_step_lcmdm_loss}
\end{equation}

The training procedure is summarized in Algorithm~\ref{alg:lcmdm_training}.

\begin{algorithm}[H]
  \caption{Training latent-conditioned masked diffusion}
  \label{alg:lcmdm_training}
  \small
  \begin{algorithmic}[1]
    \State \textbf{Input:} paired latent--sequence dataset
    $\mathcal{D}_Z=\{(z^{(i)},x^{(i)})\}_{i=1}^{N}$
    \State \textbf{Model:} latent-conditioned masked denoiser $H_\theta$
    \For{training updates}
      \State sample a minibatch $(z,x)\sim\mathcal{D}_Z$
      \State sample masking times $t\sim\mathrm{Unif}(0,1)$
      \State sample masks $m_i\sim\mathrm{Bernoulli}(t)$ independently across positions
      \State resample $m$ if $\sum_{i=1}^{T}m_i=0$
      \State construct $x_t$ by replacing positions with $m_i=1$ by $\mathsf{M}$
      \State compute logits $\ell \gets H_\theta(x_t,z)$
      \State update $\theta$ by minimizing the masked-token loss in
      Eq.~\eqref{eq:appendix_few_step_lcmdm_loss}
    \EndFor
  \end{algorithmic}
\end{algorithm}
\subsection{Sampling}
\label{app:few_step_lcmdm_sampling}

At inference time, we first sample a Gaussian latent
$z\sim\mathcal{N}(0,I)$ and initialize the sequence as fully masked,
$x^{(0)}=(\mathsf{M},\ldots,\mathsf{M})$. We then run a latent-conditioned
P2-self sampler for $K$ iterations. Let $\kappa:[0,1]\to[0,1]$ be a monotone
unmasking schedule with $\kappa(0)=0$ and $\kappa(1)=1$. At iteration $i$, we
set $t_i=i/K$ and choose the number of positions that should remain masked as
$\lfloor L(1-\kappa(t_i))\rfloor$, where $L$ is the sequence length.

Given the current partially generated sequence $x^{(i-1)}$, the denoiser
predicts logits
$\ell^{(i)}=H_\theta(x^{(i-1)},z)$. We sample a candidate clean token for each
non-fixed position using the Gumbel-max form of categorical sampling with
temperature $\tau_i$, and use the corresponding log-probability as a confidence
score. P2-self then plans the next mask set by selecting the lowest-confidence
positions to remain masked. This differs from a purely monotone unmasking
procedure: positions that were previously masked may be revealed, while positions
that were previously revealed may be remasked if their confidence is low. For
previously revealed positions, we multiply the confidence score by a remasking
strength $\eta$. Since the score is a log-probability, larger $\eta$ makes such
positions easier to remask, whereas smaller $\eta$ makes already revealed tokens
more persistent.

Thus, each iteration consists of three operations: propose candidate clean
tokens, select the low-confidence positions that should remain masked, and write
candidate tokens only into positions that were masked and are no longer selected
to remain masked. After the final step, all remaining mask tokens are filled
using the last denoiser proposal. The Gaussian latent $z$ is fixed throughout the
trajectory and provides global sequence-level conditioning at every denoising
step. When $K=1$, the method reduces to one-step latent-conditioned masked
decoding; for $K>1$, it performs iterative P2-self refinement with possible
remasking.

The full procedure is summarized in Algorithm~\ref{alg:lcmdm_sampling}.

\begin{algorithm}[H]
  \caption{Few-step latent-conditioned P2-self sampling}
  \label{alg:lcmdm_sampling}
  \small
  \begin{algorithmic}[1]
    \Require number of sampling steps $K$, sequence length $L$, mask token
    $\mathsf{M}$, denoiser $H_\theta$, unmasking schedule $\kappa$,
    temperatures $\{\tau_i\}_{i=1}^{K}$, remasking strength $\eta$
    \State sample $z\sim\mathcal{N}(0,I)$
    \State initialize $x^{(0)}\gets(\mathsf{M},\ldots,\mathsf{M})$
    \State set $f_j\gets 1$ if $x^{(0)}_j\neq \mathsf{M}$, and $f_j\gets 0$ otherwise
    \For{$i=1,\ldots,K$}
      \State $t_i\gets i/K$
      \State compute logits $\ell^{(i)}\gets H_\theta(x^{(i-1)},z)$
      \For{each position $j$ with $f_j=0$}
        \State sample Gumbel noise $g_{j,v}\sim\mathrm{Gumbel}(0,1)$ for each token $v$
        \State $\tilde{\ell}^{(i)}_{j,v}\gets \ell^{(i)}_{j,v}/\tau_i+g_{j,v}$
        \State propose $\hat{x}^{(i)}_j\gets \arg\max_v \tilde{\ell}^{(i)}_{j,v}$
        \State score $s^{(i)}_j\gets
        \log\mathrm{softmax}(\tilde{\ell}^{(i)}_{j})_{\hat{x}^{(i)}_j}$
      \EndFor
      \State set $s^{(i)}_j\gets+\infty$ for all fixed positions $j$ with $f_j=1$
      \State define previously generated positions
      $\mathcal{U}^{(i)}\gets\{j:x^{(i-1)}_j\neq\mathsf{M},\, f_j=0\}$
      \State rescale $s^{(i)}_j\gets \eta s^{(i)}_j$ for all
      $j\in\mathcal{U}^{(i)}$
      \State set $m_i\gets \lfloor L(1-\kappa(t_i))\rfloor$
      \State let $\mathcal{R}^{(i)}$ be the $m_i$ non-fixed positions with
      the lowest scores $s^{(i)}_j$
      \State set $x^{(i)}\gets x^{(i-1)}$
      \State set $x^{(i)}_j\gets\mathsf{M}$ for all $j\in\mathcal{R}^{(i)}$
      \State define newly revealed positions
      $\mathcal{A}^{(i)}\gets
      \{j:x^{(i-1)}_j=\mathsf{M}\}\setminus\mathcal{R}^{(i)}$
      \State set $x^{(i)}_j\gets \hat{x}^{(i)}_j$ for all
      $j\in\mathcal{A}^{(i)}$
    \EndFor
    \State fill any remaining masked position $j$ by setting
    $x^{(K)}_j\gets \hat{x}^{(K)}_j$
    \State \Return $x^{(K)}$
  \end{algorithmic}
\end{algorithm}

\section{Guidance algorithms}
\label{app:guidance}

This appendix gives the procedural details for the guidance mechanisms
summarized in Section~\ref{sec:guidance}. The goal is to specialize standard
guidance methods to \nameshort. Since generation consists of
sampling $z\sim\mathcal{N}(0,I)$ and decoding once, guidance acts either on the
final logits, on the latent variable before decoding, or on the generator
parameters during fine-tuning.

\begin{algorithm}[H]
\caption{Classifier-free guidance for one-step generation}
\label{alg:cfg_guidance}
\small
\begin{algorithmic}[1]
\Require generator $G_\theta$, condition $y$, guidance scale $s$, temperature $\tau$
\State sample $z\sim\mathcal{N}(0,I)$
\State compute conditional logits $\ell^c\gets G_\theta(z,y)$
\State compute unconditional logits $\ell^u\gets G_\theta(z,\varnothing)$
\State combine logits $\ell^{\mathrm{cfg}}\gets \ell^u+s(\ell^c-\ell^u)$
\State sample $\hat{x}_t\sim\mathrm{Cat}(\mathrm{softmax}(\ell^{\mathrm{cfg}}_t/\tau))$
for all positions $t$ in parallel
\State \Return $\hat{x}$
\end{algorithmic}
\end{algorithm}

\begin{algorithm}[H]
\caption{Latent-space classifier guidance}
\label{alg:latent_classifier_guidance}
\small
\begin{algorithmic}[1]
\Require generator $G_\theta$, reward $R(\cdot,y)$, condition $y$
\Require step size $\eta$, guidance steps $K$, relaxation $\rho$, sampling temperature $\tau$
\State sample $z^{0}\sim\mathcal{N}(0,I)$
\For{$k=0,\ldots,K-1$}
  \State compute logits $\ell^k\gets G_\theta(z^k,y)$
  \State compute relaxed output $\tilde{x}^k\gets \rho(\ell^k)$
  \State update $z^{k+1}\gets z^k+\eta\nabla_{z^k}R(\tilde{x}^k,y)$
\EndFor
\State compute final logits $\ell^K\gets G_\theta(z^K,y)$
\State sample $\hat{x}_t\sim\mathrm{Cat}(\mathrm{softmax}(\ell^K_t/\tau))$
for all positions $t$ in parallel
\State \Return $\hat{x}$
\end{algorithmic}
\end{algorithm}

\begin{algorithm}[H]
\caption{Reward fine-tuning for one-step generation}
\label{alg:reward_finetuning}
\small
\begin{algorithmic}[1]
\Require pretrained generator $G_{\theta_0}$, trainable generator $G_\theta$
\Require reward $R(\cdot,y)$, relaxation $\rho$
\Require weights $\lambda_{\mathrm{reward}}$, $\lambda_{\mathrm{anchor}}$
\State initialize $\theta\gets\theta_0$
\For{fine-tuning updates}
  \State sample condition $y$ and latent $z\sim\mathcal{N}(0,I)$
  \State compute relaxed output $\tilde{x}_\theta\gets \rho(G_\theta(z,y))$
  \State compute anchor penalty $\mathcal{R}(\theta,\theta_0)$
  \State update $\theta$ by minimizing
  $\mathcal{L}_{\mathrm{FT}}
  =-\lambda_{\mathrm{reward}}R(\tilde{x}_\theta,y)
  +\lambda_{\mathrm{anchor}}\mathcal{R}(\theta,\theta_0)$
\EndFor
\State \Return fine-tuned generator $G_\theta$
\end{algorithmic}
\end{algorithm}

\subsection{Relaxations and implementation choices}
\label{app:guidance_relaxations}

For classifier guidance and reward fine-tuning, the reward gradient must pass
through the model output. We use two final-layer relaxations. The soft relaxation
sets $\rho(\ell)=\mathrm{softmax}(\ell/\tau_r)$, or the corresponding sigmoid
relaxation for binary outputs, and evaluates the reward on these continuous
probabilities. This gives a low-variance gradient but introduces a mismatch
between relaxed optimization and final discrete sampling. The Gumbel-based
relaxation instead uses a Gumbel-Softmax or straight-through estimator
\citep{jang2017gumbel,maddison2017concrete}, which more closely matches the
categorical sampling path but can be noisier.

In all guidance experiments, CFG uses conditioning dropout during training and
two generator evaluations at sampling time, one conditional and one
unconditional. Latent-space classifier guidance keeps the generator fixed and
uses a small number of gradient steps on $z$ before the final decode. Reward
fine-tuning initializes from the pretrained generator and uses an anchor penalty,
implemented as a logit-level or distribution-level matching loss, to discourage
drift from the data distribution while optimizing the reward.

\section{Experimental details}
\label{app:exp_details}

This appendix collects the implementation details used for the experiments in Section~\ref{sec:experiments}. The goal is to make the retained benchmarks easy to reproduce while keeping the main paper focused on the empirical argument. Table~\ref{tab:appendix_benchmark_summary} summarizes the three tasks and their evaluation protocols, and the following subsections provide benchmark-specific training and evaluation details.

\begin{table}[t]
\centering
\caption{
Summary of the experimental benchmarks used in the paper.
}
\label{tab:appendix_benchmark_summary}
\footnotesize
\setlength{\tabcolsep}{4.5pt}
\renewcommand{\arraystretch}{1.1}
\begin{tabular}{@{}L{0.22\linewidth}L{0.29\linewidth}L{0.16\linewidth}L{0.25\linewidth}@{}}
\toprule
Benchmark & Data / setting & Metric & Protocol source \\
\midrule
MNIST-Binary & Unconditional generation on binarized $28\times 28$ MNIST digits & FID $\downarrow$ & Follows the comparison protocol used by \CFM \\
Fly Brain enhancer generation & Conditional DNA sequence generation on DeepFlyBrain, 81 target classes, sequence length 500 & FBD $\downarrow$ & Follows the enhancer-generation benchmark of Dirichlet Flow Matching \\
LM1B & Unconditional one-step language generation on LM1B, sequence length 128 & Gen-PPL $\downarrow$, entropy & Follows the LM1B evaluation style of FLM \\
\bottomrule
\end{tabular}
\end{table}

\subsection{MNIST-Binary details}
\label{app:exp_mnist}
\begin{figure}[t]
  \centering
  \includegraphics[width=0.5\linewidth]{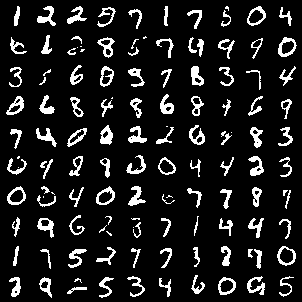}
  \caption{Unconditional MNIST-Binary generations from our model.}
  \label{fig:uncond_mnist_results}
\end{figure}
We use the standard torchvision MNIST split and binarize each pixel by thresholding the grayscale value at $0.5$. Each image is then flattened into a length-784 binary sequence for training, while evaluation samples are reshaped back to $28\times 28$ images. The main paper reports the unconditional setting (\texttt{--mode uncond}), and the code also supports a class-conditional variant.

Training uses a two-stage schedule over 200 epochs, with the first half used for Stage~A and the second half for Stage~B. Stage~A trains the binary autoencoder and latent flow, while Stage~B freezes them and trains the one-step decoder. We optimize all trainable parameters with AdamW, learning rate $2\times 10^{-4}$, weight decay $10^{-4}$, and a cosine decay schedule with one-epoch warmup. We evaluate FID with 1{,}000 generated samples against the full MNIST training set, using the \texttt{torchmetrics} implementation of Fr\'echet Inception Distance.

\begin{table}[t]
\centering
\caption{
MNIST-Binary training configuration used for the main result.
}
\label{tab:appendix_mnist_config}
\footnotesize
\setlength{\tabcolsep}{5pt}
\renewcommand{\arraystretch}{1.08}
\begin{tabular}{@{}ll@{}}
\toprule
Item & Value \\
\midrule
Dataset & MNIST, binarized at threshold 0.5 \\
Task & Unconditional one-step generation \\
Batch size & 256 \\
Epochs & 200 total (100 Stage~A, 100 Stage~B) \\
Optimizer & AdamW \\
Learning rate & $2\times 10^{-4}$ \\
Weight decay & $10^{-4}$ \\
Latent spatial size & $7\times 7$ \\
Latent channels & 16 \\
Latent noise std & 0.5 \\
Flow width / blocks / layers / heads & 128 / 5 / 5 / 4 \\
Generator width / depth & 512 / 8 \\
FID evaluation & 1{,}000 generated samples, one step \\
Random seed & 0 \\
\bottomrule
\end{tabular}
\end{table}

\subsection{Fly Brain enhancer-generation details}
\label{app:exp_enhancer}

For DNA sequence design, we follow the enhancer benchmark from Dirichlet Flow Matching and retain only the Fly Brain dataset in the paper. The dataset configuration used by our code corresponds to \texttt{DeepFlyBrain\_data.pkl}, with 81 target classes and sequence length 500 over the alphabet $\{A,C,G,T\}$. The training loader uses the benchmark train split and the default validation split (\texttt{valid} rather than \texttt{test}). Evaluation uses FBD in the embedding space of the frozen benchmark classifier, matching the protocol implemented for prior sequence-design baselines. Our code also reports a \texttt{flow\_reverse\_fbd} diagnostic for the reverse map of Stage~A, but the main paper reports the one-step generator FBD.

The Fly Brain experiment is trained for 300 epochs. We use the first 150 epochs for Stage~A and the remaining 150 for Stage~B. The run in \texttt{train\_enhancer\_fb.sh} uses batch size 128, compression ratio $0.05$ (which corresponds to latent length 25 for length-500 sequences), latent width 256, dropout 0.1, and flow depth 8 blocks. Parameters not overridden in the launch script use the defaults in \texttt{one\_step\_enhancer.py}; in particular, the optimizer is AdamW with learning rate $2\times 10^{-4}$ and weight decay $10^{-4}$. No external embedding alignment or Stage~B alignment is enabled in the reported configuration.

\begin{table}[t]
\centering
\caption{
Fly Brain enhancer-generation configuration used for the main result.
}
\label{tab:appendix_enhancer_config}
\footnotesize
\setlength{\tabcolsep}{5pt}
\renewcommand{\arraystretch}{1.08}
\begin{tabular}{@{}ll@{}}
\toprule
Item & Value \\
\midrule
Dataset & DeepFlyBrain (\texttt{fb}) \\
Sequence length / alphabet & 500 / $\{A,C,G,T\}$ \\
Number of target classes & 81 \\
Train / eval split & train / valid \\
Metric & FBD in frozen classifier embedding space \\
Batch size / eval batch size & 128 / 128 \\
Epochs & 300 total (150 Stage~A, 150 Stage~B) \\
Optimizer & AdamW \\
Learning rate / weight decay & $2\times 10^{-4}$ / $10^{-4}$ \\
Compression ratio & 0.05 \\
Latent length & 25 \\
Latent channels & 256 \\
Latent noise std & 1.0 \\
Encoder width / depth & 256 / 8 \\
Decoder depth & 8 \\
Generator width / depth & 512 / 8 \\
Flow width / blocks / layers / heads & 128 / 8 / 4 / 4 \\
Dropout & 0.1 \\
$\beta_{\mathrm{KL}}$ & 1.0 \\
Alignment losses & Off in the reported run \\
Random seed & 0 \\
\bottomrule
\end{tabular}
\end{table}

\subsection{LM1B text-generation details}
\label{app:exp_lm1b}

For text generation, we use LM1B tokenized with the \texttt{Qwen/Qwen2.5-0.5B} tokenizer. The preprocessing pipeline truncates or pads each sequence to length 128 and stores tokenized training and test splits as memory-mapped NumPy arrays. The underlying dataset loader uses \texttt{dvruette/lm1b} with the \texttt{\string~parquet} revision. All text experiments use the same two-stage training pipeline as the rest of the paper, but with a frozen pretrained language-model encoder in Stage~A and a Transformer decoder in Stage~B.

Stage~A uses a frozen Qwen2.5-0.5B encoder, a learned linear projection from 896 to 256 dimensions, a RealNVP flow with 12 coupling layers, and a lightweight Transformer decoder that reconstructs Qwen-tokenized text. The dequantization noise scale is 0.2. Stage~A is trained for 200k optimization steps with batch size 128, AdamW, learning rate $3\times 10^{-4}$, weight decay 0.01, and 2k warmup steps. The flow-weight coefficient is annealed from 0.5 to 1.0 over the first 5\% of training.

Stage~B uses a one-step Transformer generator with width 768, depth 12, 12 attention heads, and feed-forward dimension 2048. It is trained for 400k steps with batch size 32, AdamW, learning rate $2\times 10^{-4}$, weight decay 0.01, 2k warmup steps, cosine restarts every 50k steps, gradient clipping at 1.0, and EMA decay 0.999. In addition to token-level cross entropy, Stage~B uses KL distillation from a frozen Qwen2.5-0.5B causal LM with distillation weight 1.0 and temperature 2.0. During training, 30\% of updates replace the Stage~A latent with pure Gaussian noise scaled by 0.90 to better match one-step prior sampling at inference time.

For evaluation, we generate 1{,}024 samples from the one-step model, decode them with the Qwen tokenizer, and score the resulting text with GPT-2 Large. We report Gen-PPL and average unigram entropy. Table~\ref{tab:lm1b-entropy-genppl-raw} lists the raw values for the temperature sweep with top-$p$ 0.95 and latent scale 0.8, and Figure~\ref{fig:lm1b-entropy-genppl-tradeoff-app} repeats the same tradeoff plot for completeness. We also compute an LM1B reference entropy of 4.31 and a reference Gen-PPL using the same decode-and-score pipeline on real LM1B test sentences. Because our generated text is decoded with a Qwen tokenizer before scoring with GPT-2 Large, the absolute Gen-PPL values are tied to this evaluation pipeline; we therefore use the FLM comparison primarily as a relative benchmark.

\begin{table}[t]
\centering
\caption{
Raw entropy and generative perplexity values used in Figure~\ref{fig:lm1b-entropy-genppl-tradeoff} for one-step LM1B generation. For our method, we report the temperature sweep with fixed top-$p=0.95$ and latent scale $z_{\mathrm{scale}}=0.8$. Lower generative perplexity and higher entropy are preferred.
}
\label{tab:lm1b-entropy-genppl-raw}
\small
\setlength{\tabcolsep}{5.5pt}
\renewcommand{\arraystretch}{1.08}
\begin{tabular}{@{}llccc@{}}
\toprule
Category & Method & Temperature & Gen. PPL $\downarrow$ & Entropy $\uparrow$ \\
\midrule
Baseline & Duo + DCD      & --   & 1224.52 & 4.33 \\
Baseline & Duo + Di4C     & --   &  292.94 & 3.79 \\
Baseline & MDLM + SDTT    & --   & 1429.48 & 4.31 \\
Baseline & MDLM + Di4C    & --   & 1217.10 & 4.38 \\
Baseline & CFM            & --   &  269.72 & 3.10 \\
Baseline & FMLM           & --   &  119.34 & 4.16 \\
Baseline & DFM (PSD)      & --   &   94.08 & 4.06 \\
Baseline & DFM (ESD)      & --   &   68.11 & 3.79 \\
\midrule
Ours & Ours              & 0.50 &   45.82 & 3.6088 \\
Ours & Ours              & 0.60 &   76.19 & 4.0841 \\
Ours & Ours              & 0.65 &   98.96 & 4.3192 \\
Ours & Ours              & 0.70 &  134.06 & 4.5716 \\
\midrule
Reference & LM1B test set & --   &   47.47 & 4.31 \\
\bottomrule
\end{tabular}
\end{table}

\begin{figure}[t]
  \centering
  \includegraphics[width=0.82\linewidth]{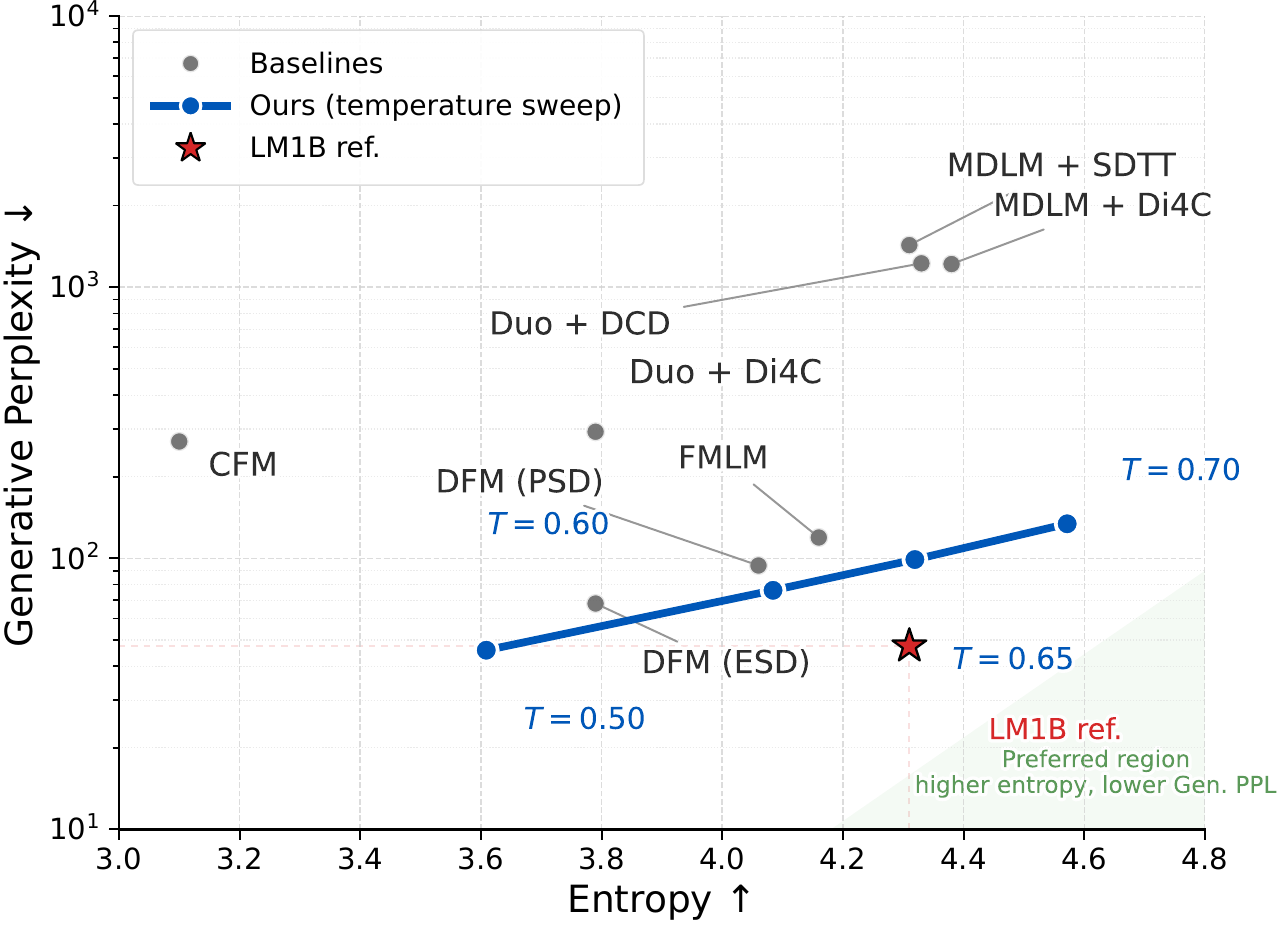}
  \caption{
  Entropy--perplexity tradeoff for one-step LM1B generation, reproduced from Figure~\ref{fig:lm1b-entropy-genppl-tradeoff}. The appendix copy is included with the evaluation details so that the sampling settings and metric definitions can be read alongside the result.
  }
  \label{fig:lm1b-entropy-genppl-tradeoff-app}
\end{figure}

\begin{table}[t]
\centering
\caption{
LM1B configuration used for the reported one-step text result.
}
\label{tab:appendix_lm1b_config}
\footnotesize
\setlength{\tabcolsep}{5pt}
\renewcommand{\arraystretch}{1.08}
\begin{tabular}{@{}ll@{}}
\toprule
Item & Value \\
\midrule
Dataset & LM1B (\texttt{dvruette/lm1b}, \texttt{\string~parquet} revision) \\
Tokenizer & Qwen2.5-0.5B tokenizer \\
Vocabulary size & 151{,}936 \\
Sequence length & 128 \\
Stage~A encoder & Frozen Qwen2.5-0.5B, hidden size 896 \\
Stage~A projection dim & 256 \\
Stage~A flow & RealNVP, 12 coupling layers, hidden dim 1024 \\
Stage~A optimizer / LR / WD & AdamW / $3\times 10^{-4}$ / 0.01 \\
Stage~A batch size / steps & 128 / 200k \\
Stage~B generator & Transformer, width 768, depth 12, heads 12, FFN 2048 \\
Stage~B optimizer / LR / WD & AdamW / $2\times 10^{-4}$ / 0.01 \\
Stage~B batch size / steps & 32 / 400k \\
KD teacher & Frozen Qwen2.5-0.5B causal LM \\
KD weight / temperature & 1.0 / 2.0 \\
Gaussian-latent mixing prob. & 0.30 \\
Inference latent scale & 0.90 \\
Evaluation samples & 1{,}024 \\
Evaluation settings & temperature 0.8, top-$p$ 0.95, GPT-2 Large scorer \\
Reported checkpoint & Stage~B step 152{,}038 \\
\bottomrule
\end{tabular}
\end{table}

\subsection{Asset licenses and provenance}
\label{app:asset_licenses}

The main assets used in the experiments are the following: MNIST, which we cite as CC BY-SA 3.0; the DeepFlyBrain model, which the Kipoi model page lists as MIT; the Fly Brain / DeepFlyBrain article, which is CC BY 4.0; the associated Zenodo design package, which is distributed under an academic non-commercial software license; LM1B benchmark code/scripts, which are Apache-2.0, while the underlying corpus license is not explicitly stated on the official benchmark page; the GPT-2 model/tokenizer, which uses OpenAI's modified MIT license; and the Qwen2.5-0.5B model/tokenizer, which is Apache-2.0.

\subsection{Compute and implementation notes}
\label{app:exp_compute}

The MNIST and Fly Brain experiments were launched through single-node SLURM jobs requesting one NVIDIA H200 GPU, 32\,GB host memory, 12 CPU cores, and a two-day wall-clock limit. The text experiments use the local training code and configuration files described above, but the current draft does not yet include a consolidated summary of text training wall-clock time or total compute. Across benchmarks, we use the default random seed 0 in the released scripts unless otherwise noted.

\newpage
\section*{NeurIPS Paper Checklist}

\begin{enumerate}

\item {\bf Claims}
    \item[] Question: Do the main claims made in the abstract and introduction accurately reflect the paper's contributions and scope?
    \item[] Answer: \answerYes{}
    \item[] Justification: The abstract and introduction state the paper's actual contributions: a two-stage model for one-step discrete generation and empirical results on MNIST-Binary, Fly Brain, and LM1B. The appendix supplies the expressivity argument for direct one-step token-wise decoders, and the paper also notes the main limitation that stronger iterative baselines remain better in harder settings (Sections~\ref{sec:method}, \ref{sec:experiments}, Appendix~\ref{app:trilemma}, and the Conclusion).

\item {\bf Limitations}
    \item[] Question: Does the paper discuss the limitations of the work performed by the authors?
    \item[] Answer: \answerYes{}
    \item[] Justification: The paper states that the model is competitive in the one-step regime but still trails the strongest iterative baselines on harder settings, and the Conclusion limits the scale of the current study to moderate-size models and benchmarks. The appendix also clarifies that the text benchmark uses a benchmark-specific evaluation pipeline and that the reported gains should be read in that context (Sections~\ref{sec:method}, \ref{sec:experiments}, \ref{app:exp_lm1b}).

\item {\bf Theory assumptions and proofs}
    \item[] Question: For each theoretical result, does the paper provide the full set of assumptions and a complete (and correct) proof?
    \item[] Answer: \answerYes{}
    \item[] Justification: The theoretical result is the expressivity barrier for direct one-step token-wise decoders. The appendix states the assumptions, defines the factorized family, and gives a complete proof by both dimension counting and an explicit counterexample (Appendix~\ref{app:trilemma}).

    \item {\bf Experimental result reproducibility}
    \item[] Question: Does the paper fully disclose all the information needed to reproduce the main experimental results of the paper to the extent that it affects the main claims and/or conclusions of the paper (regardless of whether the code and data are provided or not)?
    \item[] Answer: \answerYes{}
    \item[] Justification: The appendix specifies the datasets, splits, preprocessing, model components, training schedules, optimizers, learning rates, evaluation metrics, sampling settings, and reported checkpoints for the three main benchmarks (Appendix~\ref{app:exp_details}, especially Sections~\ref{app:exp_mnist}, \ref{app:exp_enhancer}, and \ref{app:exp_lm1b}).
    \item[] Guidelines:
    \begin{itemize}
        \item The answer NA means that the paper does not include experiments.
        \item If the paper includes experiments, a No answer to this question will not be perceived well by the reviewers: Making the paper reproducible is important, regardless of whether the code and data are provided or not.
        \item If the contribution is a dataset and/or model, the authors should describe the steps taken to make their results reproducible or verifiable. 
        \item Depending on the contribution, reproducibility can be accomplished in various ways. For example, if the contribution is a novel architecture, describing the architecture fully might suffice, or if the contribution is a specific model and empirical evaluation, it may be necessary to either make it possible for others to replicate the model with the same dataset, or provide access to the model. In general, releasing code and data is often one good way to accomplish this, but reproducibility can also be provided via detailed instructions for how to replicate the results, access to a hosted model (e.g., in the case of a large language model), releasing of a model checkpoint, or other means that are appropriate to the research performed.
        \item While NeurIPS does not require releasing code, the conference does require all submissions to provide some reasonable avenue for reproducibility, which may depend on the nature of the contribution. For example
        \begin{enumerate}
            \item If the contribution is primarily a new algorithm, the paper should make it clear how to reproduce that algorithm.
            \item If the contribution is primarily a new model architecture, the paper should describe the architecture clearly and fully.
            \item If the contribution is a new model (e.g., a large language model), then there should either be a way to access this model for reproducing the results or a way to reproduce the model (e.g., with an open-source dataset or instructions for how to construct the dataset).
            \item We recognize that reproducibility may be tricky in some cases, in which case authors are welcome to describe the particular way they provide for reproducibility. In the case of closed-source models, it may be that access to the model is limited in some way (e.g., to registered users), but it should be possible for other researchers to have some path to reproducing or verifying the results.
        \end{enumerate}
    \end{itemize}

\item {\bf Open access to data and code}
    \item[] Question: Does the paper provide open access to the data and code, with sufficient instructions to faithfully reproduce the main experimental results, as described in supplemental material?
    \item[] Answer: \answerNo{}
    \item[] Justification: The paper provides detailed reproduction instructions, but the code and any released checkpoints will be made public only after acceptance rather than being open at submission time. We therefore describe a reproducible route, but not current open access to the full artifact set (Appendix~\ref{app:exp_details}).
    \item[] Guidelines:
    \begin{itemize}
        \item The answer NA means that paper does not include experiments requiring code.
        \item Please see the NeurIPS code and data submission guidelines (\url{https://nips.cc/public/guides/CodeSubmissionPolicy}) for more details.
        \item While we encourage the release of code and data, we understand that this might not be possible, so "No" is an acceptable answer. Papers cannot be rejected simply for not including code, unless this is central to the contribution (e.g., for a new open-source benchmark).
        \item The instructions should contain the exact command and environment needed to run to reproduce the results. See the NeurIPS code and data submission guidelines (\url{https://nips.cc/public/guides/CodeSubmissionPolicy}) for more details.
        \item The authors should provide instructions on data access and preparation, including how to access the raw data, preprocessed data, intermediate data, and generated data, etc.
        \item The authors should provide scripts to reproduce all experimental results for the new proposed method and baselines. If only a subset of experiments are reproducible, they should state which ones are omitted from the script and why.
        \item At submission time, to preserve anonymity, the authors should release anonymized versions (if applicable).
        \item Providing as much information as possible in supplemental material (appended to the paper) is recommended, but including URLs to data and code is permitted.
    \end{itemize}

\item {\bf Experimental setting/details}
    \item[] Question: Does the paper specify all the training and test details (e.g., data splits, hyperparameters, how they were chosen, type of optimizer, etc.) necessary to understand the results?
    \item[] Answer: \answerYes{}
    \item[] Justification: The appendix gives the dataset splits, preprocessing, training schedules, optimizer settings, latent sizes, flow depths, generator widths, evaluation budgets, and checkpoint choices for all three benchmarks. The remaining implementation details are also summarized in tables for each benchmark (Appendix~\ref{app:exp_details}).

\item {\bf Experiment statistical significance}
    \item[] Question: Does the paper report error bars suitably and correctly defined or other appropriate information about the statistical significance of the experiments?
    \item[] Answer: \answerNo{}
    \item[] Justification: The Fly Brain table reports $\pm$ values for the main one-step result, but the paper does not yet define those error bars or report comparable uncertainty information for all main results. We therefore do not claim a full statistical-significance treatment in the current draft.
    \item[] Guidelines:
    \begin{itemize}
        \item The answer NA means that the paper does not include experiments.
        \item The authors should answer "Yes" if the results are accompanied by error bars, confidence intervals, or statistical significance tests, at least for the experiments that support the main claims of the paper.
        \item The factors of variability that the error bars are capturing should be clearly stated (for example, train/test split, initialization, random drawing of some parameter, or overall run with given experimental conditions).
        \item The method for calculating the error bars should be explained (closed form formula, call to a library function, bootstrap, etc.)
        \item The assumptions made should be given (e.g., Normally distributed errors).
        \item It should be clear whether the error bar is the standard deviation or the standard error of the mean.
        \item It is OK to report 1-sigma error bars, but one should state it. The authors should preferably report a 2-sigma error bar rather than state that they have a 96\% CI, if the hypothesis of Normality of errors is not verified.
        \item For asymmetric distributions, the authors should be careful not to show in tables or figures symmetric error bars that would yield results that are out of range (e.g. negative error rates).
        \item If error bars are reported in tables or plots, the authors should explain in the text how they were calculated and reference the corresponding figures or tables in the text.
    \end{itemize}

\item {\bf Experiments compute resources}
    \item[] Question: For each experiment, does the paper provide sufficient information on the computer resources (type of compute workers, memory, time of execution) needed to reproduce the experiments?
    \item[] Answer: \answerNo{}
    \item[] Justification: The appendix specifies the H200-based SLURM setup for MNIST and Fly Brain, but the text experiment section still lacks a consolidated wall-clock and total-compute summary. As written, the paper does not yet provide compute details at the level requested for every experiment (Appendix~\ref{app:exp_compute}).
    
\item {\bf Code of ethics}
    \item[] Question: Does the research conducted in the paper conform, in every respect, with the NeurIPS Code of Ethics \url{https://neurips.cc/public/EthicsGuidelines}?
    \item[] Answer: \answerYes{}
    \item[] Justification: The work uses standard public benchmarks and pretrained models, does not involve human subjects, and does not report any data collection or deployment practice that conflicts with the NeurIPS Code of Ethics. The paper's stated limitations and benchmark scope are consistent with the ethics guidance.

\item {\bf Broader impacts}
    \item[] Question: Does the paper discuss both potential positive societal impacts and negative societal impacts of the work performed?
    \item[] Answer: \answerYes{}
    \item[] Justification: The checklist identifies the main positive and negative impacts: the method may reduce inference cost for discrete generation and help sequence-design workflows, but the same capability can also lower the cost of generating synthetic text or biological sequences for misuse. The paper frames the method as a generative modeling advance rather than a safety mechanism or a deployment system with built-in controls.
    
\item {\bf Safeguards}
    \item[] Question: Does the paper describe safeguards that have been put in place for responsible release of data or models that have a high risk for misuse (e.g., pretrained language models, image generators, or scraped datasets)?
    \item[] Answer: \answerNo{}
    \item[] Justification: The paper does not describe gated release, usage restrictions, or safety filters for models or data. That is consistent with the current status of the project, which is a research paper rather than a deployed release.

\item {\bf Licenses for existing assets}
    \item[] Question: Are the creators or original owners of assets (e.g., code, data, models), used in the paper, properly credited and are the license and terms of use explicitly mentioned and properly respected?
    \item[] Answer: \answerNo{}
    \item[] Justification: The paper credits the benchmark sources and now records the licenses/terms of the major assets in Appendix~\ref{app:asset_licenses}, but the LM1B corpus itself does not have a clearly specified upstream license on the official benchmark page. Because that uncertainty remains, we do not mark the licensing disclosure as fully complete.

\item {\bf New assets}
    \item[] Question: Are new assets introduced in the paper well documented and is the documentation provided alongside the assets?
    \item[] Answer: \answerNA{}
    \item[] Justification: The paper does not introduce a new dataset or benchmark package as a released asset. It introduces a model and reports experiments on existing public benchmarks.

\item {\bf Crowdsourcing and research with human subjects}
    \item[] Question: For crowdsourcing experiments and research with human subjects, does the paper include the full text of instructions given to participants and screenshots, if applicable, as well as details about compensation (if any)? 
    \item[] Answer: \answerNA{}
    \item[] Justification: The paper does not include crowdsourcing studies or other human-subjects experiments.

\item {\bf Institutional review board (IRB) approvals or equivalent for research with human subjects}
    \item[] Question: Does the paper describe potential risks incurred by study participants, whether such risks were disclosed to the subjects, and whether Institutional Review Board (IRB) approvals (or an equivalent approval/review based on the requirements of your country or institution) were obtained?
    \item[] Answer: \answerNA{}
    \item[] Justification: The paper does not involve human-subjects research, so IRB approval is not applicable.

\item {\bf Declaration of LLM usage}
    \item[] Question: Does the paper describe the usage of LLMs if it is an important, original, or non-standard component of the core methods in this research? Note that if the LLM is used only for writing, editing, or formatting purposes and does not impact the core methodology, scientific rigor, or originality of the research, declaration is not required.
    \item[] Answer: \answerNA{}
    \item[] Justification: The core method does not rely on an LLM as an original or non-standard component. The pretrained language models used in the text experiment are standard frozen components for representation extraction and evaluation, not the core methodological contribution.

\end{enumerate}

\end{document}